\documentclass[12pt]{article}

\usepackage{fullpage}

\usepackage{multirow,bigdelim}

\usepackage{graphicx} 
\usepackage{subfigure}

\usepackage[boxruled,lined]{algorithm2e}

\usepackage{amsmath,amssymb}
\usepackage{amsthm}

\newcommand{\bx}{\boldsymbol{x}}

\newcommand{\bu}{\boldsymbol{u}}
\newcommand{\by}{\boldsymbol{y}}

\newcommand{\sY}{\mathcal{Y}}

\newcommand{\bz}{\boldsymbol{z}}

\newcommand{\bw}{\boldsymbol{w}}

\newcommand{\bv}{\boldsymbol{v}}
\newcommand{\bzero}{\boldsymbol{0}}
\newcommand{\balpha}{\boldsymbol{\alpha}}
\newcommand{\bbeta}{\boldsymbol{\beta}}

\newcommand{\sX}{\mathcal{X}}

\newcommand{\sK}{\mathcal{K}}

\DeclareMathOperator*{\argmin}{argmin}

\newcommand{\field}[1]{\mathbb{#1}}

\newcommand{\E}{\field{E}}
\renewcommand{\Pr}{\field{P}}

\newcommand{\Ind}[1]{ \field{I}\left\{{#1}\right\} }
\newcommand{\norml}[1]{\left\|{#1}\right\|}
\newcommand{\norm}[1]{\|{#1}\|}

\renewcommand{\ss}{\subseteq}

\newtheorem{lemma}{Lemma}
\newtheorem{theorem}{Theorem}
\newtheorem{cor}{Corollary}

\newtheorem{definition}{Definition}

\newcommand{\secref}[1]{Section~\ref{#1}}
\newcommand{\subsecref}[1]{Subsection~\ref{#1}}
\newcommand{\thmref}[1]{Theorem~\ref{#1}}
\newcommand{\corref}[1]{Corollary~\ref{#1}}
\newcommand{\lemref}[1]{Lemma~\ref{#1}}

\usepackage{color}
\newcommand{\sW}{\mathcal{W}}
\newcommand{\sH}{\mathcal{H}}
\newcommand{\sD}{\mathcal{D}}
\newcommand{\reals}{\mathbb{R}}
\newcommand{\inner}[1]{\left\langle #1 \right\rangle}
\newcommand{\be}{\boldsymbol{e}}
\newcommand{\bsigma}{\boldsymbol{\sigma}}
\newcommand{\Ocal}{\mathcal{O}}

\begin{document}

\title{On the Complexity of Learning with Kernels}
\author{
Nicol\`o Cesa-Bianchi\thanks{
Universit\`a degli Studi di Milano, Italy 
(\texttt{nicolo.cesa-bianchi@unimi.it}) }
\and Yishay
Mansour\thanks{Microsoft Research and Tel-Aviv University, Israel
(\texttt{mansour@tau.ac.il}). Supported in part by a grant from the Israel
Science Foundation, a grant from the United States-Israel Binational Science
Foundation (BSF), a grant by Israel Ministry of Science and Technology and
the Israeli Centers of Research Excellence (I-CORE) program (Center No.\
4/11).}
\and Ohad Shamir \thanks{
Weizmann Institute of Science, Israel (\texttt{ohad.shamir@weizmann.ac.il}).
Supported in part by an Israeli Science Foundation grant (no. 425/13) and an
FP7 Marie Curie CIG grant.
}
}

\date{}

\maketitle

\begin{abstract}
A well-recognized limitation of kernel learning is the requirement to handle a kernel matrix, whose size is quadratic in the number of training examples. Many methods have been proposed to reduce this computational cost, mostly by using a subset of the kernel matrix entries, or some form of low-rank matrix approximation, or a random projection method. In this paper, we study lower bounds on the error attainable by such methods as a function of the number of entries observed in the kernel matrix or the rank of an approximate kernel matrix. We show that there are kernel learning problems where no such method will lead to non-trivial computational savings. Our results also quantify how the problem difficulty depends on parameters such as the nature of the loss function, the regularization parameter, the norm of the desired predictor, and the kernel matrix rank. Our results also suggest cases where more efficient kernel learning might be possible.
\end{abstract}

\section{Introduction}\label{sec:introduction}

We consider the well-known problem of kernel learning (see, e.g.,~\cite{scholkopf2001learning}), where given a training set of labeled examples
$\{(\bx_t,y_t)\}_{t=1}^{m}$ from a product domain $\sX\times \sY$, our goal
is to find a linear predictor $\bw$ in a reproducing kernel Hilbert space
which minimizes the average loss, possibly with some regularization.
Formally, our goal is to solve
\begin{equation}\label{eq:erm}
\min_{\bw\in\sW}\frac{1}{m}\sum_{t=1}^{m}\ell(\inner{\bw,\psi(\bx_t)},y_t)+\frac{\lambda}{2}\norm{\bw}^2\;,
\end{equation}
where $\sW$ is a convex subset of some reproducing kernel Hilbert space
$\sH$, $\psi:\sX\mapsto \sH$ is a feature mapping to the Hilbert space,
$\ell$ is a loss function convex in its first argument, and $\lambda\geq 0$
is a regularization parameter. For example, in the standard formulation of
Support Vector Machines, we take $\ell$ to be the hinge loss, pick some
$\lambda>0$, and let $\sW$ be the entire Hilbert space. Alternatively, one
can also employ hard regularization, e.g., setting $\lambda=0$ and taking
$\sW=\{\bw:\norm{\bw}\leq R\}$.

It is well-known that even if $\sH$ is high or infinite dimensional, we can
solve \eqref{eq:erm} in polynomial time, provided there is an efficiently
computable \emph{kernel function} $k$ such that
$k(\bx,\bx')=\inner{\psi(\bx),\psi(\bx')}$. The key insight is provided by
the representer theorem, which implies that an optimum of \eqref{eq:erm}
exists in the span of $\psi(\bx_1),\ldots,\psi(\bx_m)$. Therefore, instead of
optimizing over $\bw$, we can optimize over a coefficient vector $\balpha$,
which implicitly specifies a predictor via
$\bw(\balpha)=\sum_{j=1}^{m}\alpha_j\psi(\bx_j)$. In this case,
\eqref{eq:erm} reduces to
\[
\min_{\balpha\,:\,\bw(\balpha)\in\sW}\frac{1}{m}\sum_{t=1}^{m}
\ell\left(\sum_{j=1}^{m}\alpha_j\inner{\psi(\bx_j),\psi(\bx_t)},y_t\right)+\frac{\lambda}{2}\norm{\bw(\balpha)}^2~.
\]
Defining the $m\times m$ kernel matrix
$K_{i,j}=\inner{\psi(\bx_i),\psi(\bx_j)}=k(\bx_i,\bx_j)$, we can re-write the
above as
\begin{equation}\label{eq:kernelerm}
\min_{\balpha\,:\,\bw(\balpha)\in \sW}\frac{1}{m}\sum_{t=1}^{m}\ell\left(\balpha^\top K \be_t,y_t\right)+\frac{\lambda}{2}\balpha^\top K \balpha~.
\end{equation}
This is a convex problem, which can generally be solved in polynomial time.
The resulting $\balpha$ implicitly defines the linear predictor
$\bw(\balpha)$ in the Hilbert space: Given a new point $\bx$ to predict on,
this can be efficiently done according to
\[
\inner{\bw(\balpha),\psi(\bx)} = \inner{\sum_{j=1}^{m}\alpha_j\psi(\bx_j),\psi(\bx)} = \sum_{j=1}^{m}\alpha_j\inner{\psi(\bx_j),\psi(\bx)} = \sum_{j=1}^{m}\alpha_j k(\bx_j,\bx)~.
\]
Unfortunately, a major handicap of kernel learning is that it requires
computing and handling an $m\times m$ matrix, where $m$ is the size of the
training data, and this can be prohibitive in large-data applications. This
has led to a large literature on efficient kernel learning, which attempts to
reduce its computational complexity. As far as we know, the algorithms
proposed so far fall into one or more of the following categories
(see below for specific references):
\begin{itemize}
  \item \emph{Limiting the number of kernel evaluations:} A dominant
      computational bottleneck in kernel learning is computing all
      entries of the kernel matrix. Thus, several algorithms attempt to
      learn using a much smaller number of kernel evaluations -- either by
      sampling them or using other schemes which require ``reading'' only a
      small part of the kernel matrix.
  \item \emph{Low-Rank Kernel Approximation:} Instead of using the full
      $m\times m$ kernel matrix, one can use instead a low-rank
      approximation of it. Learning with a low-rank matrix can be done in a
      computationally much more efficient manner than with a general
      kernel matrix (e.g., \cite{scholkopf2001learning,Bach13}).
  \item \emph{Projection to a low-dimensional space:} Each instance $\bx$
      is mapped to a finite-dimensional vector
          $\phi(\bx)=\bigl(\phi_1(\bx),\ldots,\phi_d(\bx)\bigr)$ where
          $d\ll m$, so that $\inner{\phi(\bx),\phi(\bx')}\approx
          k(\bx,\bx')$. Note that this is equivalent to a kernel problem
          where the rank of the kernel matrix is $d$, so it can be seen as
          a different kind of low-rank kernel approximation technique.
\end{itemize}
Existing theoretical results focus on performance guarantees for various
algorithms. In this work, we consider a complementary question, which
surprisingly has not been thoroughly explored (to the best of our knowledge):
What are the inherent obstacles to efficient kernel learning? For example, is
it possible to reduce the number of kernel evaluations while maintaining the
same learning performance? Is there always a price to pay for low-rank matrix
approximation? Can finite-dimensional projection methods match the
performance of algorithms working on the original kernel matrix?

Specifically, we study information-theoretic lower bounds on the attainable
performance, measured in terms of optimization error on a given training set.
We consider two distinct types of constraints:
\begin{itemize}
  \item The number of kernel evaluations (or equivalently, the number of
      entries of the kernel matrix observed) is bounded by $B$, where $B$
      is generally assumed to be much smaller than $m^2$ (the number of
      entries in the kernel matrix).
  \item The algorithm solves \eqref{eq:kernelerm}, but using some low-rank
      matrix $\hat{K}$ instead of $K$. This can be seen as using a low-rank
      kernel matrix approximation.
\end{itemize}
We make no assumptions whatsoever on which kernel evaluations are used, or
the type of low-rank approximation, so our results apply to all the methods
mentioned previously, and any future potential method which uses these types
of approaches. We note that although we focus on optimization error on a
given training set, our lower bounds can also be potentially extended to
generalization error, where the data is assumed to be sampled i.i.d.\ from an
underlying distribution. We discuss this point further in
\secref{sec:discussion}.

Our first conclusion, informally stated, is that it is generally impossible to
make kernel learning more efficient in a non-trivial manner. For example,
suppose we have a budget $B$ on the number of kernel evaluations, where $B\ll
m^2$. Then the following ``trivial'' sub-sampling method turns out to be
optimal in general: Sub-sample $\sqrt{B}$ examples from the training data
uniformly at random (throwing away all other examples), compute the full
$\sqrt{B}\times\sqrt{B}$ kernel matrix based on the sub-sample, and train a
predictor using this matrix. This is an extremely n\"{a}ive algorithm,
throwing away almost all of the data, yet we show that there are cases where
no algorithm can be substantially better. Another pessimistic result can be
shown for the low-rank matrix approximation approach: There are cases where
any low-rank approximation will impact the attainable performance.

Our formal results go beyond these observations, and quantify the attainable
performance as a function of several important problem parameters, such as
the kernel matrix rank, regularization parameter, norm of the desired
predictor, and the nature of the loss function. In particular:
\begin{itemize}
  \item Given a kernel evaluation budget constraint $B$:
  \begin{itemize}
  \item For the absolute loss, no regularization ($\lambda=0$), and a
      constant norm constraint on the domain, we have an error lower
      bound of $\Omega(B^{-1/4})$. A matching upper bound is obtained by
      the sub-sampling algorithm discussed earlier.
  \item For soft regularization (with regularization parameter $\lambda > 0$ and no norm constraint), we attain error lower bounds which depend on the structure of the loss function. Some particular corollaries include:
    \begin{itemize}
        \item For the absolute loss, $\Omega(1/\lambda\sqrt{B})$. Again, a matching upper bound is attained by a sub-sampling algorithm.
        \item For the hinge loss, $\Omega(1)$ as long as $B< 1/\lambda^2$. Although it only applies in a certain budget regime, it is tight in terms of identifying the kernel evaluation budget required to make the error sub-constant. Moreover, it sheds some light on previous work (e.g.,~\cite{CoShSr12}) which considered efficient kernel learning methods for the hinge loss.
        \item For the squared loss, $\Omega\left(\min\left\{1,\lambda\sqrt{B}\right\}\right)^{-3}$, as long as $B\ll m^2$. Like the result for the other losses, it implies that no sub-constant error is possible unless $B\geq 1/\lambda^2$.
    \end{itemize}
  \end{itemize}
  \item For learning with low-rank approximation, with rank parameter $d$, in the case of Ridge Regression (squared loss and soft regularization), we attain an error lower bound of $\Omega((\lambda d)^{-3})$. Thus, to get sub-constant error, we need the rank to scale at least like $1/\lambda$.
\end{itemize}

The role of the loss function is particularly interesting, since it has not been well-recognized in previous literature, yet our results indicate that it may play a key role in
the complexity of kernel learning. For example, as we discuss in
\secref{sec:budget}, efficient kernel learning is trivial with the linear
loss, harder for smooth and non-linear losses, and appears to be especially
hard for non-smooth losses. Our results also highlight the importance of the
kernel matrix rank in determining the difficulty of kernel learning. While it
has been recognized that low rank can make kernel learning easy (see
references below), our results formally establish the reverse direction,
namely that (some) high-rank matrices are indeed hard to learn with any
algorithm.

\subsection*{Related Work}

The literature on efficient kernel methods is vast and we cannot do it full justice. A
few representative examples include sparse greedy kernel approximations
\cite{scholkopf2001learning}, Nystr\"{o}m-based methods, which sample a few
rows and columns and use it to construct a low-rank approximation
\cite{DrMa05,kumar2009sampling}, random finite-dimensional kernel
approximations such as random kitchen sinks
\cite{RaRe07,RaRe08,dai2014scalable}, the kernelized stochastic batch
Perceptron for learning with few kernel evaluations \cite{CoShSr12},
the random budget Perceptron and the Forgetron \cite{cavallanti2007tracking,dekel2008forgetron}, divide-and-conquer approaches
\cite{ZDW13,hsieh2014divide}, sequential algorithms with early stopping
\cite{yao2007early,RaWaYu14}, other numerical-algebraic methods for low-rank
approximation, e.g.,
\cite{fine2002efficient,shawe2004kernel,BaJo05,mahoney2009cur,kumar2009sampling},
combinations of the above \cite{dai2014scalable}, and more. Several works
provide a theoretical analysis on the performance of such methods, as a
function of the rank, number of kernel evaluations, dimensionality of the
finite-dimensional space, and so on. Beyond the works mentioned above, a few
other examples include
\cite{cortes2010impact,yang2012nystrom,Bach13,lin2014sample}.

In terms of lower bounds, we note that there are existing results on the
error of matrix approximation, based on partial access to the matrix (see
\cite{bar2003sampling,frieze2004fast}). However, the way the error is
measured is not suitable to our setting, since they focus on the Frobenius
norm of $K-\widehat{K}$, where $K$ is the original matrix and $\widehat{K}$ is the
approximation. In contrast, in our setting, we are interested in the error of
a resulting predictor rather than the quality of matrix approximation.
Therefore, even if $\norm{K-\widehat{K}}$ is large, it could be that
$\widehat{K}$ can still be used to learn an excellent predictor. Another distinct
line of work studies how to reduce the complexity of a kernel predictor at
test time, e.g., by making it supported on a few support vectors (see for
instance \cite{cotter2013learning} and references therein). This differs from
our work, which focuses on efficiency at training time.

\subsection*{Paper Organization}

Our paper is organized as follows. In \secref{sec:hard}, we introduce the
class of kernel matrices which shall be used to prove our results, and
discuss how they can be generated by standard kernels. In
\secref{sec:budget}, we provide lower bounds in a model where the algorithm
is constrained in terms of the number of kernel evaluations used. We consider
this model in two flavors, one where there is a norm constraint and no
regularization (\subsecref{subsec:norm}), and one where there is regularization
without norm constraint (\subsecref{subsec:reg}). In the former case, we focus
on a particular loss, while in the latter case, we provide a more general
result and discuss how different types of losses lead to different types of
lower bounds. In \secref{sec:lowrank}, we consider the model where the
algorithm is constrained to use a low-rank kernel matrix approximation.
We conclude and discuss open questions in \secref{sec:discussion}. Proofs
appear in Appendix \ref{sec:proofs}.

\section{Hard Kernel Matrices}\label{sec:hard}

For our results, we utilize a set $\sK_{d,m}$ of ``hard'' kernel matrices,
which are essentially permutations of block-diagonal $m\times m$ matrices
with at most $d$ blocks. More formally:
\begin{definition}
  Let $\sK'_{d,m}$ be the class of all block-diagonal $m\times m$ matrices,
  composed of at most $d$ blocks, with entry values of $1$ within each block.
  We define $\sK_{d,m}$ to be all matrices which belong to $\sK'_{d,m}$
  under some permutation $\pi:\{1\ldots m\}\mapsto \{1\ldots m\}$ of their rows and columns:
\[
\sK_{d,m}=\left\{K\in \{0,1\}^{m\times m}:~ \exists~\pi,K'\in \sK'_{d,m}\text{s.t.}~\forall i,j\in \{1\ldots m\},~K_{i,j}=K'_{\pi(i),\pi(j)}\right\}~.
\]
\end{definition}
From the definition, it is immediate that any $K\in\sK_{d,m}$ is positive
semidefinite (and hence is a valid kernel matrix), with rank at most $d$.
Moreover, the magnitude of the diagonal elements is at most $1$, which means
that our data lies in the unit ball in the Hilbert space.

Since our focus is on generic kernel learning, it is sufficient to consider
this class in order to establish hardness results. However, it is still
worthwhile to consider what kernels can induce this class of kernel matrices.
A sufficient condition can be quantified via the following lemma.
\begin{lemma}\label{lem:suffkern}
Suppose there exist $\bz_1,\ldots,\bz_d\in \sX$ such that
$k(\bz_i,\bz_j)=\Ind{\bz_i=\bz_j}$. Then any $K\in \sK_{d,m}$ is induced by
some $m$ instances $\{\bx_t\}_{t=1}^{m}\in \sX$.
\end{lemma}
The proof is immediate: Given any $K$, for any block $i$ of size $n_i$,
create $n_i$ copies of $\bz_i$, and order the instances according to the
relevant permutation.

It is straightforward to see that \lemref{lem:suffkern} holds for linear
kernels $k(\bx,\bx')=\inner{\bx,\bx'}$ and for homogeneous polynomial kernels
$k(\bx,\bx')=\inner{\bx,\bx'}^p$. It also holds (approximately) for Gaussian
kernels $k(\bx,\bx')=\exp(-\norm{\bx-\bx'}^2/\gamma)$ if there exist $d$
equi-distant points in $\sX$, where the squared distance is much larger than
$\gamma$. In that case, instead of $0$ outside the blocks, we will have
$\epsilon$ where $\epsilon$ is exponentially small, and can be shown to be
negligible for our purposes.

However, a close inspection of our results reveals that they are in fact
applicable to a much larger class of matrices: All we truly require is to
have $\bz_1,\ldots,\bz_d\in \sX$ such that $k(\bz_i,\bz_i)=a$ and
$k(\bz_i,\bz_j)=c$ for some distinct constants $a,c$ for all $i\neq j$. This
condition holds for most kernels we are aware of. For example, if there are
$d$ equi-distant points $\bz_1,\ldots,\bz_d\in \sX$, then this condition is
fulfilled for any shift-invariant kernel (where $k(\bx,\bx')$ is some
function of $\norm{\bx-\bx'}$). Similarly, if there are $d$ points
$\bz_1,\ldots,\bz_d$ which have the same inner product, then the condition is
fulfilled for any inner product kernel (where $k(\bx,\bx')$ is some function
of $\inner{\bx,\bx'}$). In order to keep a more coherent presentation we will
concentrate here on the Boolean case defined previously, where $a=1$ and
$c=0$.

Although our formal results and proofs contain many technical details, their basic intuition is quite simple:
When $d$ is sufficiently large, any matrix in $\sK_{d,m}$ is of high rank, and cannot be approximated well by any low-rank matrix. Therefore, under suitable conditions, no low-rank matrix approximation approach
can work well. Moreover, when $d$ is large, then the kernel matrix is quite sparse, and contains a large number
of relatively small blocks. Thus, for an appropriate randomized choice of a matrix in $\sK_{d,m}$, any
algorithm with a limited budget of kernel evaluations will find it difficult to detect these blocks. With
a suitable construction, we can reduce the kernel optimization problem to that of detecting these blocks, from
which our results follow.

\section{Budget Constraints}\label{sec:budget}

We now turn to present our results for the case of budget constraints. In this setting, the
learning algorithm is given the target values $y_1,\ldots,y_m$, but not the
kernel matrix $K$. Instead, the algorithm may \emph{query} at most $B$
entries in the kernel matrix (where $B$ is a user-defined positive integer),
and then return a coefficient vector based on this information. This model
represents approaches which attempt to reduce the computational complexity of
kernel learning by reducing the number of kernel evaluations needed. Standard
learning algorithms essentially require $B=\Omega(m^2)$, and the goal is to
learn to similar accuracy with a budget $B\ll m^2$. In this section, we
discuss the inherent limitations of this approach.

\subsection{Norm Constraint, Absolute Loss}\label{subsec:norm}

We begin by demonstrating a lower bound using the absolute loss
$\ell_1(u,y)=|u-y|$ on the domain $\sW=\bigl\{\bw \,:\, \norm{\bw}^2\leq 2\bigr\}$ (or
equivalently, coefficient vectors $\balpha$ satisfying $\balpha^\top
K\balpha\leq 2$), and our goal is to minimize the average loss, which
equals
\[
\min_{\balpha\,:\,\balpha^\top K \balpha\leq 2} \frac{1}{m}\sum_{t=1}^{m}
\left|\balpha^\top K \be_t-y_t\right|~.
\]
\begin{theorem}\label{thm:absmain}
For any rank parameter $d$, any sample size $m\geq 2^{7}d$, any budget size
$B<\frac{3}{50}d^2$, and for any budgeted algorithm, there exists a kernel
matrix $K\in \sK_{2d,m}$ and target values $y_1,\ldots,y_m\in [-1,+1]$, such
that the returned coefficient vector $\balpha$ satisfies
\begin{equation}\label{eq:absmaineq}
\left(\frac{1}{m}\sum_{t=1}^{m}
\left|\balpha^\top K \be_t-y_t\right|\right)
-\min_{\balpha\,:\,\balpha^\top K \balpha\leq 2}\frac{1}{m}\sum_{t=1}^{m}
\left|\balpha^\top K \be_t-y_t\right|
\geq \frac{1}{70\sqrt{d}}~.
\end{equation}
\end{theorem}
The proof and the required construction appears in \subsecref{subsec:budget}.
Note that the algorithm is allowed to return any coefficient vector (not
necessarily one satisfying the domain constraint $\balpha^\top K \balpha \leq
2$).

The theorem provides a lower bound on the attainable error, for any rank
parameter $d$ and assuming the sample size $m$ and budget $B$ are in an
appropriate regime. A different way to phrase this is that if $B$ is
sufficiently smaller than $m^2$, then we can find some $d$ on the order of
$\sqrt{B}$, such that \thmref{thm:absmain} holds. More formally:
\begin{cor}\label{cor:sqrt4B}
  There exist universal constants $c,c'>0$ such that if $B\leq cm^2$, there is
  an $m\times m$ kernel matrix $K$ (belonging to $\sK_{2d,m}$ for some appropriate $d$)
  and target values in $[-1,+1]$ such that the returned
  coefficient vector $\balpha$ satisfies
  \[
  \left(\frac{1}{m}\sum_{t=1}^{m}
  \left|\balpha^\top K \be_t-y_t\right|\right)
  -\min_{\balpha\,:\,\balpha^\top K \balpha\leq 2}\frac{1}{m}\sum_{t=1}^{m}
  \left|\balpha^\top K \be_t-y_t\right|
  \geq \frac{c'}{\sqrt[4]{B}}~.
  \]
\end{cor}
In words, the attainable error given a budget of $B$ cannot go down faster
than $1/\sqrt[4]{B}$. Next, we show that this is in fact the optimal rate,
and is achieved by the following simple strategy:
\begin{enumerate}
\item Given a training set $\{(\bx_t,y_t)\}_{t=1}^{m}$ of size $m$, sample
    $\lfloor\sqrt{B}\rfloor$ training examples uniformly at random (with
    replacement), getting
    $\{(\bx_{t_j},y_{t_j})\}_{j=1}^{\lfloor\sqrt{B}\rfloor}$.
\item Compute the kernel matrix $\widehat{K}\in \reals^{\lfloor\sqrt{B}\rfloor
    \times \lfloor\sqrt{B}\rfloor}$ defined as
    $\widehat{K}_{j,j'}=k(\bx_{t_j},\bx_{t_{j'}})$, using at most $B$ queries.
\item Solve the kernel learning problem on the sampled set, getting a
    coefficient vector $\widehat{\balpha}$:
    \[
    \min_{\widehat{\balpha}\,:\,\widehat{\balpha}^\top \hat{K} \widehat{\balpha}\leq 2}\left(\frac{1}{\lfloor\sqrt{B}\rfloor}\sum_{j=1}^{\lfloor\sqrt{B}\rfloor}
    \left|\widehat{\balpha}^\top \widehat{K} \be_j-y_{t_j}\right|\right)~.
    \]
\item Return the coefficient vector $\balpha$ such that
    $\alpha_{t_j}=\widehat{\alpha}_j$ for $j=1,\ldots,\lfloor \sqrt{B}\rfloor$,
    and $\alpha_t=0$ otherwise.
\end{enumerate}
Essentially, this strategy approximately solves the original problem by
drawing a subset of the training data ---small enough so that we can compute
its kernel matrix in full--- and solving the learning problem on that data.
Since we use a sample of size $\lfloor \sqrt{B}\rfloor$, then by standard
generalization guarantees for learning bounded-norm predictors using
Lipschitz loss functions (e.g., \cite{kakade2009complexity}), we get a generalization error
upper bound of
\[
\Ocal\left(\frac{1}{\sqrt{\lfloor \sqrt{B}\rfloor}}\right)=\Ocal\left(\frac{1}{\sqrt[4]{B}}\right)
\]
which matches the lower bound in \corref{cor:sqrt4B} up to constants.

To summarize, we see that with the absolute loss, given a constraint on the
number of kernel evaluations, there exist no better method than throwing away
most of the data, and learning on a sufficiently small subset. Moreover, any
method using a non-trivial budget (significantly smaller than $m^2$) must
suffer a performance degradation.

\subsection{Soft Regularization, General Losses}\label{subsec:reg}

Having obtained an essentially tight result for the absolute loss, it is
natural to ask what can be obtained for more generic losses. To study this
question, it will be convenient to shift to the setting where the domain
$\sW$ is the entire Hilbert space, and we use a regularization term.
Following \eqref{eq:kernelerm}, this reduces to solving
\[
\min_{\balpha}\frac{1}{m}\sum_{t=1}^{m}\ell\left(\balpha^\top K \be_t,y_t\right)+\frac{\lambda}{2}\balpha^\top K \balpha~.
\]

We start by defining the main quantity we are interested in,
\[
\Delta_\ell(m,\alpha,K,\lambda,y)=\left(\frac{1}{m}\sum_{t=1}^{m}\ell\left(\balpha^\top K \be_t,y_t\right)+\frac{\lambda}{2}\balpha^\top K \balpha\right)
    - \min_{\balpha}\left(\frac{1}{m}\sum_{t=1}^{m}\ell\left(\balpha^\top K \be_t,y_t\right)+\frac{\lambda}{2}\balpha^\top K \balpha\right),
\]
where $\ell$ is a loss function.

First, we provide a general result, which applies to any non-negative loss
function, and then draw from it corollaries for specific losses:
\begin{theorem}\label{thm:lossmain}
Suppose the loss function $\ell$ is non-negative. For any rank parameter $d$,
any sample size $m\geq 2^{7}d$, any budget $B<\frac{3}{50}d^2$, and for any
budgeted algorithm, there exists a kernel matrix $K\in \sK_{2d,m}$ and target
values $y_1,\ldots,y_m$ in $\sY$, such that the returned coefficient vector
$\balpha$ satisfies
\begin{align}
 \label{eq:lossmaineq}
\Delta_\ell(m,\alpha,K,\lambda,y) \geq
    \frac{1}{60}\lambda d~\min_{p\in \left[\frac{1}{2},2\right]}\max_{y\in\sY}(2u_1^*-u_2^*)^2
\end{align}
where
\[
u^*_1 = \argmin_{u} \ell(u,y)+p\lambda d u^2
\qquad\text{and}\qquad
u_2^* = \argmin_{u}\ell(u,y)+\frac{p\lambda d}{2} u^2~.
\]
\end{theorem}
The proof and the required construction appears in \subsecref{subsec:general}.

Roughly speaking, to get a non-trivial bound, we need the loss to be such
that when the regularization parameter is order of $\lambda d$, then scaling
it by a factor of $2$ changes the location of the optimum $u^*$ by a factor
different than $2$. For instance, this rules out linear losses of the form
$\ell(u,y)=yu$. For such a loss, we have
\[
u^{*}_1=\arg\min_u \frac{n}{m}yu+\lambda u^2 = -\frac{ny}{2\lambda m}
\qquad\text{and}\qquad
u^{*}_2=\arg\min_u \frac{n}{m}yu+\frac{\lambda}{2} u^2 = -\frac{ny}{\lambda m}~.
\]
Thus we get that $(2u_1^*-u_2^*)^2=0$ and the lower bound is trivially $0$.
While this may seem at first like an unsatisfactory bound, in fact this
should be expected: For linear loss and no domain constraints, we don't need
to observe the kernel matrix $K$ at all in order to find the optimal
solution! To see this, note that the optimization problem in
\eqref{eq:kernelerm} reduces to
\[
\min_{\balpha}\frac{1}{m}\sum_{t=1}^{m}y_t \balpha^\top K \be_t+\frac{\lambda}{2}\balpha^\top K \balpha
\]
or, equivalently,
\[
\min_{\balpha}~\balpha^\top K \bv+\frac{\lambda}{2}\balpha^\top K \balpha
\]
where $\bv$ is a known vector and $K$ is the partially-unknown kernel matrix.
Differentiating the expression by $\balpha$ and equating to $\bzero$, getting
\[
K\bv+\lambda K \balpha = \bzero~.
\]
Thus, an optimum of this problem is simply $-\frac{1}{\lambda}\bv$, regardless
of what is $K$. This shows that for linear losses, we can find the optimal
predictor with zero queries of the kernel matrix.

Thus, the kernel learning problem is non-trivial only for non-linear losses,
which we now turn to examine in more detail.

\subsubsection{Absolute Loss}

First, let us consider again the absolute loss in this setting. We easily get
the following corollary of \thmref{thm:lossmain}:

\begin{cor}\label{cor:abs}
Let $\ell_1(u,y)=|u-y|$ be the absolute loss. There exist universal constants
$c,c'>0$, such that if $B\leq cm^2$, then for any budgeted algorithm there
exists an $m\times m$ kernel matrix $K$ and target values $y$ such that
$\Delta_{\ell_1}(m,\alpha,K,\lambda,y)$
is lower bounded by $\frac{c'}{\lambda\sqrt{B}}$.
\end{cor}
\begin{proof}
  To apply \thmref{thm:lossmain}, let us compute $(2u_1^*-u_2^*)^2$, where we use the particular choice $y=\frac{1}{2p\lambda d}$.
  It is readily verified that $u_1^*=u_2^*=y=\frac{1}{2p\lambda d}$, leading
  to the lower bound
\begin{align*}
  \frac{1}{60}\lambda d~ \min_{p\in \left[\frac{1}{2},2\right]}(u_1^*)^2
  &=   \frac{1}{60}\lambda d~ \min_{p\in \left[\frac{1}{2},2\right]}\left(\frac{1}{2p \lambda d}\right)^2 \\
  &= \frac{1}{60}\lambda d\left(\frac{1}{4\lambda d}\right)^2=\frac{1}{960\lambda d}~.
\end{align*}
  In particular, suppose we choose $d=\left\lceil \sqrt{\frac{100}{3}B}~\right\rceil$. Then
  we get a lower bound of $\frac{c'}{\lambda \sqrt{B}}$ for
  $c'=2^{-13}$. The conditions of \thmref{thm:lossmain} are satisfied if $m\geq 2^7d = 2^7\left\lceil \sqrt{\frac{100}{3}B}~\right\rceil$
  and  $B< \frac{3}{50}d^2 = \frac{3}{50}\left(\left\lceil
  \sqrt{\frac{100}{3}B}~\right\rceil\right)^2$. The latter always
  holds, whereas the former is indeed true if $B$ is smaller than $cm^2$ for
  $c=2^{-20}$.
\end{proof}
As in the setting of \thmref{thm:absmain}, this lower bound is tight, and we
can get a matching $\Ocal\bigl(1/\lambda \sqrt{B}\bigr)$ upper bound by learning with a
random sub-sample of $\Theta\bigl(\sqrt{B}\bigr)$ training examples, using
generalization bounds for minimizers of strongly-convex and Lipschitz
stochastic optimization problems \cite{sridharan2009fast}.


Note that, unlike our other lower bounds, \corref{cor:abs} is proven using a different choice of $y$ for each $\lambda$.
It is not clear whether this requirement is real, or is simply an artifact of our proof technique.

\subsubsection{Hinge Loss}

Intuitively, the proof of \corref{cor:abs} relied on the absolute loss having
a non-smooth ``kink'' at $\Theta(1/\lambda \sqrt{B})$, which prevented the
optimal $u_1^*,u_2^*$ from moving as a result of the changed regularization
parameter. Results of similar flavor can be obtained with any other loss
which has an optimum at a non-smooth point. However, when we do not control
the location of the ``kink'' the results may be weaker. A good example is the
hinge loss, $\ell_h(u,y)=\max\{0,1-uy\}$, which is non-differentiable at the
fixed location $p=1$:
\begin{cor}\label{cor:hinge}
Let $\ell_h(u,y)=\max\{0,1-uy\}$ be the hinge loss. There exist universal
constants $c,c',c''>0$, such that if $\lambda\leq \frac{1}{4}$ and $B <
\frac{c}{\lambda^2} \le c' m^2$, then for any budgeted algorithm, there exist
an $m\times m$ kernel matrix $K$ and target values $y$ in $\{-1,+1\}$ such that
$\Delta_{\ell_h}(m,\alpha,K,\lambda,y)$
is lower bounded by $c''$.
\end{cor}
\begin{proof}
  To apply \thmref{thm:lossmain}, let us compute $(2u_1^*-u_2^*)^2$, where we use the particular choice
  $y=1$. It is readily verified that $u_1^*=u_2^*=1$, as long as $p\lambda d\leq 1$, and is certainly satisfied for any $p\in \left[\frac{1}{2},2\right]$
  by assuming $\lambda d\leq \frac{1}{2}$.
  Therefore, if $\lambda d\leq \frac{1}{2}$, then in \thmref{thm:lossmain},
  we get $u_1^*=u_2^*=1$, and thus a lower bound of
  $\frac{1}{60}\lambda d \cdot 1^2~=~ \frac{1}{60}\lambda d$.

In particular, suppose we pick $d=\left\lfloor\tfrac{1}{2\lambda}\right\rfloor$. Since we assume $\lambda\leq \tfrac{1}{4}$, this means that the lower bound above is $\frac{1}{60}\lambda \left\lfloor\tfrac{1}{2\lambda}\right\rfloor \ge \tfrac{1}{240}=c''$. The conditions of \thmref{thm:lossmain} are satisfied if $m\geq 2^7 d = 2^7\left\lfloor\tfrac{1}{2\lambda}\right\rfloor$ and $B < \frac{3}{50} d^2=\frac{3}{50}\left(\left\lfloor\tfrac{1}{2\lambda}\right\rfloor\right)^2<\frac{3}{200}\frac{1}{\lambda^2}=\frac{c}{\lambda^2}<
\frac{c}{2^{10}}m^2=c'm^2$, which are indeed implied by the corollary's conditions.
\end{proof}
Unlike the bound for the absolute loss, here the result is weaker, and only
quantifies a regime under which sub-constant error is impossible. In
particular, the condition $B = \Ocal\bigl(\tfrac{1}{\lambda^2}\bigr)$ is not interesting for
constant $\lambda$. However, in learning problems $\lambda$ usually scales
down as $m^{-q}$ where $q\geq 1/2$ and often $q=1$. In that case, we get
constant error as long as $B\ll m^{2q}$, which establishes that learning is
impossible for a budget smaller than a quantity in the range from $\Omega(m)$ to $\Omega(m^2)$, depending on the value of $q$. For $q=1$, that is $\lambda=\Theta(1/m)$, learning is impossible
without querying a constant fraction of the kernel matrix.

Moreover, it is possible to show that our lower bound is tight, in terms of
identifying the threshold for making the error sub-constant. As before, we
consider the strategy of sub-sampling $\sqrt{B}$ training examples and
learning with respect to the induced kernel matrix. Since we use $\sqrt{B}$
examples and $\lambda$-strongly convex regularization, the expected error
scales as $\frac{1}{\lambda\sqrt{B}}$ \cite{sridharan2009fast}. This is
sub-constant in the regime $B=\omega(1/\lambda^2)$, and matches our lower
bound. We emphasize that when $B$ is $\omega(1/\lambda^2)$, we do
not have a non-trivial lower bound, and it remains an open problem to
understand what can be attained for the hinge loss in this regime.

Another interesting consequence of the corollary is the required budget as a
function of the norm of a ``good predictor'' we want to compete with. In
\cite{CoShSr12}, several algorithmic approaches have been studied, which were
all shown to require $\Omega(\norm{\bu}^4)$ kernel evaluations to be
competitive with a given predictor $\bu$, even in the ``realizable'' case
where the predictor attains zero average hinge loss. An examination of the
proof of theorem 2 reveals that the construction is such that there exists a
predictor $\bu$ which attains zero hinge loss on all the examples, and whose
norm\footnote{To see this, recall that we use a block-diagonal kernel matrix
composed of at most $2d$ all-ones blocks, and where $y=1$ always. So by
picking $\alpha_t=1/n_i$ for any index $t$ in block $i$ (where $n_i$ is the
size of the block), we get zero hinge loss, and the norm is
$\sqrt{\balpha^\top K\balpha}\leq \sqrt{\sum_{i} 1} \leq \sqrt{2d}$.
Moreover, in the proof of \corref{cor:hinge} we pick $d=\left\lfloor
\frac{1}{2\lambda}\right\rfloor$, so the norm is
  $\Ocal(1/\sqrt{\lambda})$.}
is $\Ocal(1/\sqrt{\lambda})$. \corref{cor:hinge} shows that the budget must
be at least $\Omega(1/\lambda^2)=\Omega(\norm{\bu}^4)$ to get sub-constant
error in the worst case. Although our setting is slightly different than
\cite{CoShSr12}, this provides evidence that the
$\Omega(\norm{\bu}^4)$ bounds in \cite{CoShSr12} are tight in
terms of the norm dependence.

\subsubsection{Squared Loss}

In the case of absolute loss and hinge loss, the results depend on a
non-differentiable point in the loss function. It is thus natural to conclude
by considering a smooth differentiable loss, such as the squared loss:
\begin{cor}\label{cor:squared}
Let $\ell_2(u,y)=(u-y)^2$ be the squared loss. There exist universal constants
$c,c'>0$, such that
\begin{itemize}
  \item If $1\leq B\leq \frac{1}{\lambda^2}\leq cm^2$, then
      for any budgeted algorithm there exists an $m\times m$ kernel matrix
      and target values in $[-1,+1]$ such that
      $\Delta_{\ell_2}(m,\alpha,K,\lambda,y)$ is
      lower bounded by $c'$.
  \item If $\frac{1}{\lambda^2}\leq B\leq cm^2$, then for any budgeted
      algorithm there exists an $m\times m$ kernel matrix and target values
      in $[-1,+1]$ such that
      $\Delta_{\ell_2}(m,\alpha,K,\lambda,y)$ is lower bounded by
      $c'(\lambda \sqrt{B})^{-3}$.
\end{itemize}
\end{cor}
This lower bound is weaker than the $\Omega(1/\lambda \sqrt{B})$ lower bound
attained for the absolute loss. This is essentially due to the smoothness of
the squared loss, and we do not know if it is tight. In any case, it proves
that even for the squared loss, at least $1/\lambda^2$ kernel evaluations are
required to get sub-constant error. In learning problems, where $\lambda$
often scales down as $m^{-q}$ (where $q\geq 1/2$ and often $q=1$), we get a
required budget size of $m^{2q}$. This is super-linear when $p>1/2$, and
becomes $m^2$ when $q=1$ -- in other words, we need to compute a constant
portion of the entire $m\times m$ kernel matrix.

\begin{proof}
  To apply \thmref{thm:lossmain}, let us compute $(2u_1^*-u_2^*)^2$. It is readily verified that
  $u_1^*=\frac{y}{1+p\lambda d}$ and $u_2^*=\frac{y}{1+p\lambda d/2}$, leading
  to the lower bound
  \begin{align*}
  &\frac{1}{60}\lambda d~ \min_{p\in \left[\frac{1}{2},2\right]}
  \max_{y\in\sY}\left(\frac{2y}{1+p\lambda d}-\frac{y}{1+p\lambda d/2}\right)^2\\
  &=\frac{1}{60}\lambda d~ \min_{p\in \left[\frac{1}{2},2\right]}\max_{y\in\sY}\left(\frac{y}{\left(1+p\lambda d\right)\left(1+\frac{p\lambda d}{2}\right)}\right)^2\\
  &\geq \frac{1}{60}\lambda d~\min_{p\in \left[\frac{1}{2},2\right]}\max_{y\in\sY}\left(\frac{y}{\left(1+p\lambda d\right)^2}\right)^2.
  \end{align*}
  Taking in particular $y=1$, we get
  \begin{equation}\label{eq:sqlowbound}
  \frac{1}{60}\lambda d~ \min_{p\in \left[\frac{1}{2},2\right]}\left(\frac{1}{\left(1+p\lambda d\right)^2}\right)^2 ~\geq~ \frac{1}{60}\frac{\lambda d}{(1+2\lambda d)^4}\;.
  \end{equation}
  We now consider two ways to pick $d$, corresponding to the two cases considered in the corollary:
  \begin{itemize}
   \item If $1\leq B\leq \frac{1}{\lambda^2}$, we pick $d=\left\lceil
       \sqrt{\frac{100}{3\lambda^2}}~\right\rceil$. Since $\lambda\leq
       1$, we have $d\lambda <7$, and this means that \eqref{eq:sqlowbound} is bounded below by
       $c'=2^{-18}$. The conditions of
       \thmref{thm:lossmain} are satisfied if $m\geq 2^7 d = 2^7\left\lceil
       \sqrt{\frac{100}{3\lambda^2}}~\right\rceil$ and
        $B<\frac{3}{50}d^2=\frac{3}{50}\left(\left\lceil
       \sqrt{\frac{100}{3\lambda^2}}~\right\rceil\right)^2$. These are
       satisfied by assuming $B\leq \frac{1}{\lambda^2}\leq c m^2$
       for 
       $c=2^{-20}$.
    \item If $B\geq \frac{1}{\lambda^2}$, we pick $d=\left\lceil
        \sqrt{\frac{100}{3}B}~\right\rceil$. Plugging this into
        \eqref{eq:sqlowbound} and using the assumption $B\geq
        \frac{1}{\lambda^2}$ (or equivalently, $\lambda \sqrt{B}\geq 1$),
        we get a lower bound of $c'\frac{\lambda \sqrt{B}}{(\lambda
        \sqrt{B})^4} = c'(\lambda \sqrt{B})^{-3}$ for an appropriate
        constant $c'=2^{-18}<1/960$. Moreover, the conditions of
        \thmref{thm:lossmain} are satisfied if $m\geq 2^7d = 2^7\left\lceil
        \sqrt{\frac{100}{3}B}~\right\rceil$ and
        $B<\frac{3}{50}d^2=\frac{3}{50}\left(\left\lceil
        \sqrt{\frac{100}{3}B}~\right\rceil\right)^2$. The latter always
        holds, whereas the former indeed holds if $B$ is less than $c
        m^2$ for 
        $c=2^{-20}$.
  \end{itemize}
This completes the proof.
\end{proof}

\section{Low-Rank Constraints}\label{sec:lowrank}

In this section, we turn to discuss the second broad class of approaches,
which replace the original kernel matrix $K$ by a low-rank approximation
$K'$. As explained earlier, many rank-reduction approaches -- including Nystr\"{o}m method and random features -- use a low-rank approximation $K'$ with entries defined by $K'_{t,t'}=\inner{\phi(\bx_t),\phi(\bx_{t'})}$, where $\{(\bx_t,y_t)\}_{t=1}^{m}$ is the training set and $\phi:\sX\mapsto \reals^d$ is a given feature mapping, typically depending on the data.

The next result shows a lower bound on the error for any such low-rank approximation method when the algorithm used for learning is kernel Ridge Regression (i.e., when we use the squared loss and employ soft regularization).

\begin{theorem}\label{thm:lowrank}
Suppose there exist a kernel function $k$ on $\sX$ and $2d$ points $\bv_1,\ldots,\bv_{2d}\in\sX$ such that $k(\bv_i,\bv_j)=\Ind{\bv_i=\bv_j}$. Then there exists a training set $\{(\bx_t,y_t)\}_{t=1}^{m} \in \bigl(\sX\times\{-1,+1\}\bigr)^{2d}$, with corresponding kernel matrix $K$, such that for any feature mapping $\phi:\sX\mapsto \reals^d$ (possibly depending on the training set), the coefficient vector $\balpha$ returned by the Ridge Regression algorithm operating on the matrix $K'$ with entries $K'_{t,t'}=\inner{\phi(\bx_t),\phi(\bx_{t'})}$ satisfies
\begin{align*}
\Delta_{\ell_2}(m,\alpha,K,\lambda,y)
&\geq \frac{1}{2(\lambda d)^2(1+\lambda d)}
\end{align*}
where $d$ is any upper bound on the rank of $K'$ such that $2d$ divides $m$.
\end{theorem}
When $\lambda d\geq 1$, we get a $\Omega\bigl((\lambda d)^{-3}\bigr)$ bound. This bears
similarities to the bound in \corref{cor:squared}, which considered the
squared loss in the budgeted setting, where $d$ is replaced by $\sqrt{B}$
(i.e., $\Omega\bigl((\lambda \sqrt{B})^{-3}\bigr)$ when $\lambda\sqrt{B}\geq 1$). The
bound implies that to get sub-constant error, the rank required must be
larger than $1/\lambda$. When $\lambda$ itself scales down with the sample
size $m$, we get that the required rank grows with the sample size. When
$\lambda=1/m$, the required rank is $\Omega(m)$, which means that any
low-rank approximation scheme (where $d\ll m$) will lead to constant error.
As in the case of \corref{cor:squared}, we do not know whether our lower
bound is tight.

\section{Discussion and Open Questions}\label{sec:discussion}

In this paper, we studied fundamental information-theoretic barriers to
efficient kernel learning, focusing on algorithms which either limit the
number of kernel evaluations, or use a low-rank kernel matrix approximation.
We provided several results under various settings, highlighting the
influence of the kernel matrix rank, regularization parameter, norm
constraint and nature of the loss function on the attainable performance.

For general losses and kernel matrices, our conclusion is generally
pessimistic. In particular, when the number of kernel evaluations is bounded,
there are cases where no algorithm attains performance better than a trivial
sub-sampling strategy, where most of the data is thrown away. Also, no
algorithm can work well when the regularization parameter is sufficiently
small or the norm constraint is sufficiently large. On a more optimistic
note, our lower bounds are substantially weaker when dealing with smooth
losses. Although we do not know if these weaker lower bounds are tight, they
may indicate that better kernel learning algorithms are possible by
exploiting smoothness of the loss. Smoothness of the squared loss has been
used in \cite{ZDW13}, but perhaps this property can be utilized more
generally.

In our results, we focused on the problem of minimizing regularized training
error on a given training set. This is a different goal than minimizing
generalization error in a stochastic setting, where the data is assumed to be
drawn i.i.d.\ from some underlying distribution. However, we believe that our
lower bounds should also be applicable in terms of optimizing the risk (or
expected error with respect to the underlying distribution). The main
obstacle is that our lower bounds are proven for a given class of kernel
matrices, which are not induced by an explicit i.i.d.\ sampling process of training
instances. However, inspecting our basic construction in
\subsecref{subsec:budget}, it can be seen that it is very close to such a process:
The kernel is constructed by \emph{pairs} of instances sampled
i.i.d.\ from a finite set $\{\bv_i^1,\bv_i^2\}_{i=1}^{d}$. We believe that all
our results would hold if the instances were to be sampled i.i.d.\ from
$\{\bv_i^1,\bv_i^2\}_{i=1}^{d}$. The reason that we sample pairs is purely
technical, since it ensures that for every $i$, there is an equal number of
$\bv_i^1$ and $\bv_i^2$ in the training set, making the calculations more
tractable. Morally, the same techniques should work with i.i.d.\ sampling, as
long as the probability of sampling $\bv_i^1$ and $\bv_i^2$ are the same for
all $i$.

Our work leaves several questions open. First, while the results for the
absolute loss are tight, we do not know if this is the case for our other
results. Second, the low-rank result in \secref{sec:lowrank} applies only to
squared loss (Ridge Regression), and it would be interesting to extend it to
other losses. Third, it should be possible to extend our results also to
randomized algorithms that query the kernel matrix a number of times bounded
by $B$ only in expectation (with respect to the algorithm's internal
randomization), rather than deterministically. Finally, our results may
indicate that at least for smooth losses, better kernel learning algorithms
are possible, and remain to be discovered.

\subsection*{Acknowledgements}
This research was carried out in part while the authors were attending the
research program on the Mathematics of Machine Learning, at the Centre de
Recerca Matem\`atica of the Universitat Aut\`onoma de Barcelona (Spain).
Partial support is gratefully acknowledged.

\bibliographystyle{plain}
\bibliography{mybib}

\appendix

\section{Proofs}\label{sec:proofs}

\subsection{Construction properties from \secref{sec:budget}}\label{subsec:budget}

We consider a randomized strategy, where the kernel matrix is sampled
randomly from $\sK_{2d,m}$ (according to a distribution $\sD$ to be defined
shortly), and $y_1,\ldots,y_m$ are fixed deterministically in a certain way.
We will analyze what is the best possible performance using any budgeted
algorithm, in expectation over this strategy.

To define the distribution $\sD$, we let $\be_1,\ldots,\be_{2d}$ be the
standard basis vectors in $\reals^{2d}$, and sample a kernel matrix from
$\sD$ as follows:
\begin{itemize}
  \item Pick $\bsigma\in\{0,1\}^d$ uniformly at random.
  \item For all $i\in \{1,\ldots,d\}$, define $\bv_i^1,\bv_i^2$ as
  \begin{itemize}
    \item $\bv_{i}^1=\bv_{i}^2=\be_i$ if $\sigma_i=1$,
    \item $\bv_i^1=\be_i$ and $\bv_i^2=\be_{i+d}$ if $\sigma_i=0$.
  \end{itemize}
  \item For $j=1,\ldots,m/2$, choose $(\bz_{2j-1},\bz_{2j})$ uniformly at
      random from $\{(\bv_i^1,\bv_i^2)\}_{i=1}^{d}$.
  \item Choose a permutation $\pi:\{1,\ldots, m\}\mapsto \{1,\ldots, m\}$
      uniformly at random.
  \item Return the kernel matrix $K$ defined as
      $K_{i,j}=\inner{\bz_{\pi(i)},\bz_{\pi(j)}}$ for all $i,j=1,\ldots,m$.
\end{itemize}
To understand the construction, we begin by noting that $K$ represents the
inner product of a set of vectors, and hence is always positive semidefinite
and a valid kernel matrix. Moreover, $\bz_1,\ldots,\bz_m$ are all in the set
$\{\be_1,\ldots,\be_{2d}\}$, and therefore the resulting kernel matrix equals
(up to permutation of rows and columns) a block-diagonal matrix of the
following form:
  \begin{center}
  \begin{tabular}{|ccccccc|}
    \hline
    $\mathbf{1}$ & $S_1$ & $\mathbf{0}$ & $\mathbf{0}$ & $\cdots$ & $\mathbf{0}$ &  $\mathbf{0}$\\
    $S_1^\top$ & $\mathbf{1}$ & $\mathbf{0}$ & $\mathbf{0}$ & $\cdots$ & $\mathbf{0}$ & $\mathbf{0}$\\
    $\mathbf{0}$ & $\mathbf{0}$ & $\mathbf{1}$ & $S_2$ & $\cdots$ & $\mathbf{0}$ & $\mathbf{0}$\\
    $\mathbf{0}$ & $\mathbf{0}$ & $S_2^\top$ & $\mathbf{1}$ & $\cdots$ & $\mathbf{0}$ & $\mathbf{0}$\\
    $\vdots$ & $\vdots$ & $\vdots$ & $\vdots$ & $\ddots$ & $\vdots$ & $\vdots$ \\
    $\mathbf{0}$ & $\mathbf{0}$ & $\mathbf{0}$ & $\mathbf{0}$ & $\cdots$ & $\mathbf{1}$ & $S_d$\\
    $\mathbf{0}$ & $\mathbf{0}$ & $\mathbf{0}$ & $\mathbf{0}$ & $\cdots$ & $S_d^\top$ & $\mathbf{1}$\\
    \hline
  \end{tabular}
  \end{center}
Here, $S_i$ is an all-zero block if $\sigma_i=0$, and an all-ones block if
$\sigma_i=1$. In other words, the matrix is composed of $d$ blocks, one for
each value of $i=1,\ldots,d$. If $\sigma_i=1$, then block $i$ is a monolithic
all-ones block (corresponding to $\be_i$), and if $\sigma_i=0$, then block
$i$ is composed of two equal-sized sub-blocks (corresponding to $\be_i$ and
to $\be_{d+i}$). This implies that the kernel matrix is indeed in
$\sK_{2d,m}$.

Our proofs rely on the following intuition: To achieve small error, the
learning algorithm must know the values of the entries in
$S_1,S_2,\ldots,S_d$ (i.e., the values of $\bsigma$). However, when $d$ is
large, these blocks are rather small, and their entries are randomly permuted
in the matrix. Thus, any algorithm with a constrained query budget is likely
to ``miss'' many of these blocks.

To simplify the presentation, we will require a few auxiliary definitions.
First, given a kernel matrix $K\in \sK_{2d,m}$ constructed as above, let
\[
T_{i,1} = \{\pi(t):\bz_t=\bv_i^1\} \qquad\text{and}\qquad
T_{i,2} = \{\pi(t):\bz_t=\bv_i^2\}
\]
denote the set of row/column indices in the kernel matrix, corresponding to
instances which were chosen to be $\bv_i^1$ (respectively $\bv_i^2$). Note
that $\{T_{i,1},T_{i,2}\}_{i=1}^{d}$ is a disjoint partition of all indices
$\{1,\ldots,m\}$, and $|T_{i,1}|=|T_{i,2}|$. We then define,
\begin{equation}\label{eq:ndef}
T_i = T_{i,1}\cup T_{i,2} \qquad\mbox{ and }\qquad N_{i}=|T_{i}|,
\end{equation}
and also define,
\begin{equation}\label{eq:betadef}
\beta_{i,1}=\sum_{t\in T_{i,1}}\alpha_t \qquad\mbox{ and }\qquad
\beta_{i,2}=\sum_{t\in T_{i,2}}\alpha_t,
\end{equation}
to be the sum of the corresponding coefficients in the solution $\balpha$
returned by the algorithm. With these definitions, we can re-write the
average loss and the regularization term as follows.
\begin{lemma}\label{lem:beta}
For any coefficient vector $\balpha$,
\begin{align*}
\frac{1}{m}\sum_{t=1}^{m}\ell\left(\balpha^\top K \be_i,y\right) &=
\sum_{i=1}^{d}\frac{N_i}{2m}\bigl(\ell(\beta_{i,1}+\sigma_i \beta_{i,2},y)
+\ell(\sigma_i\beta_{i,1}+\beta_{i,2},y)\bigr)
\mbox{       and}\\
\balpha^\top K \balpha &= \sum_{i=1}^{d}\left(\beta_{i,1}^2+\beta_{i,2}^2+2\sigma_i \beta_{i,1}\beta_{i,2}\right),
\end{align*}
where $\beta_{i,1},\beta_{i,2}$ are defined in \eqref{eq:betadef}.
\end{lemma}
The proof is a straightforward exercise based on the definition of $K$.
Finally, we define $E_i$ to be the event that the algorithm never queries a
pair of inputs in $T_i$, i.e., the algorithm's queries
$(s_1,r_1),\ldots,(s_B,r_B)$ on the kernel matrix satisfy
\[
   s_t\not\in T_i \,\vee\, r_t\not\in T_i \qquad t=1,\dots,B~.
\]
To prove our results, we will require two key lemmas, presented below, which
quantify how any budgeted algorithm is likely to ``miss'' many blocks, and hence
have its output relatively insensitive to $\bsigma$.
\begin{lemma} \label{lem:block}
Suppose $m \ge 2d$ and $B < \tfrac{3}{50}d^2$. Then for any
deterministic learning algorithm,
\[
    \sum_{i=1}^d \Pr(E_i) > \frac{d}{2}~.
\]
\end{lemma}
The formal proof is provided below. Although it is quite technical, the
lemma's intuition is very simple: Recall that the kernel matrix is composed
of $d$ blocks, each of size $\frac{m}{d}\times \frac{m}{d}$ in expectation.
Thus, if we choose an entry uniformly at random, the chance of ``hitting''
some block is approximately $d\frac{(m/d)^2}{m^2}= \frac{1}{d}$. Thus, if we
sample $B$ points uniformly at random, where $B\ll d^2$, then the number of
``missed'' blocks $\sum_{i=1}^{d}\Ind{E_i}$ is likely to be $\Omega(d)$. The
lemma above simply quantifies this, and shows that this holds not just for
uniform sampling, but for any algorithm with a budgeted number of queries.
\begin{proof}
Recall that each $T_i$ corresponds to one of $d$ blocks in the kernel matrix
(possibly composed of two sub-blocks).
The algorithm queries $(s_1,r_1),\ldots,(s_B,r_B)$. For each possible query at time $t$
we define the set $Q_{s,t}$ of blocks such that $s$ was queried with a member of that block
and we obtained a value zero in the kernel matrix. Namely,
\begin{align*}
    Q_{s,t} = & \Bigl\{ i=1,\dots,d \,:\, (\exists \tau < t) \; s_\tau=s \wedge r_\tau\in T_i\wedge\, K_{s_{\tau},r_{\tau}} = 0  \,  \Bigr\}\\
    & \cup \Bigl\{ i=1,\dots,d \,:\, (\exists \tau < t) \; r_\tau=s \wedge s_\tau\in T_i\wedge\, K_{s_{\tau},r_{\tau}} = 0  \,  \Bigr\}~.
\end{align*}
Given the query $(s_t,r_t)$ we define $L_t=Q_{s_t,t}$ to be the blocks in which some member was queried with $s_t$, and $R_t=Q_{r_t,t}$
the blocks in which some member was queried with $r_t$.


We introduce a quantity $P_t$ defined as follows: $P_t = d+1$ if there is a
query $t' < t$ such that $K_{s_{t'},r_{t'}} = 1$ and, moreover, $s_t =
s_{t'}$ or $r_t = r_{t'}$ (that is, the block of $s_t$ or the block of $r_t$
was already discovered). Otherwise, let $P_t =
\max\bigl\{|L_t|,|R_t|\bigr\}$.

Let $D_t$ be the event that the $t$-th query discovers a new block. That is,
$D_t$ is true if and only if $K_{s_t,r_t} = 1$ and $P_t < d+1$. Using this
notation,
\begin{align}
\label{eq:clubound}
    \sum_{i=1}^d \Ind{\neg E_i}
&=
    \sum_{t=1}^B \Ind{D_t}
=
    \sum_{t=1}^B \Ind{D_t \,\wedge\, P_t < d/2} + \underbrace{\sum_{t=1}^B \Ind{D_t \,\wedge\, P_t \ge d/2}}_{N}~.
\end{align}
We will now show that unless $B\geq \tfrac{3}{50}d^2$, we can upper bound $N$
deterministically by $\sqrt{2B}$. We do this by considering separately the case $N \le \tfrac{d}{2}$ and the
case $N
> \tfrac{d}{2}$:
\begin{itemize}
\item Assume first $N > \tfrac{d}{2}$, and let $t_1,\dots,t_N$ be the
    times $t_k$ such that $\Ind{D_{t_k} \,\wedge\, P_{t_k} \ge d/2} = 1$.
    Now fix some $k$ and note that, because the common block to which
    $s_{t_k}$ and $r_{t_k}$ both belong is discovered, neither $s_{t_k}$
    nor $r_{t_k}$ can occur in a future query $(s_t,r_t)$ that discovers a
    new block. Therefore, in order to have $\Ind{D_t \,\wedge\, P_t \ge
    d/2} = 1$ for $N > \tfrac{d}{2}$ times, at least
\[
    \frac{d}{2} + \left(\frac{d}{2} - 1\right) + \dots + 1
\]
queries must be made, where each term $\tfrac{d}{2} - k + 1$ accounts for the
fact that each one of the previous $k-1$ discovered blocks might contribute
with at most a query to making $P_t \ge \tfrac{d}{2}$. So, it must be
\[
    B \ge \sum_{k=1}^{d/2} \left(\frac{d}{2} - (k - 1)\right) \geq \frac{d^2}{8}
\]
queries to discover the first $\frac{d}{2}$ blocks, which contradicts the
lemma's assumption that $B\leq \tfrac{3}{50}d^2$. Therefore,  $N \le \tfrac{d}{2}$.
\item Assume that $N \le \tfrac{d}{2}$. Using the same logic as before,
    in order to have $\Ind{D_t \,\wedge\, P_t \ge
    d/2} = 1$ for $N \le \tfrac{d}{2}$ times, at least
\[
    \frac{d}{2} + \left(\frac{d}{2} - 1\right) + \dots + \left(\frac{d}{2} - N + 1\right)
\]
queries must be made.
So, it must be
\[
    B \ge \sum_{k=1}^N \left(\frac{d}{2} - (k - 1)\right) = (d+1)\frac{N}{2} - \frac{N^2}{2}
\]
or, equivalently, $N^2 -(d+1)N + 2B \ge 0$. Solving this quadratic
inequality for $N$, and using the lemma's assumption that $N\le
\tfrac{d}{2}$, we have that $N\leq \frac{(d+1)-\sqrt{(d+1)^2-8B}}{2}$.
Using the lemma's assumption that $B\leq \tfrac{3}{50}d^2$  we get that
$N\leq \sqrt{2B}$.
\end{itemize}

We now bound the first term of~(\ref{eq:clubound}) in expectation. For any
time $t$ and query $(s_t,r_t)$, we say that $s_t$ is paired with $r_t$ if
$K_{s_t,r_t} = \inner{\bz_{2j-1},\bz_{2j}}$ for some $j\in\{1,\dots,m/2\}$,
where $\{s_t,r_t\} \equiv \{\pi(2j-1),\pi(2j)\}$. Clearly, $\Pr(\text{$s_t$
paired with $r_t$}\bigr) = \tfrac{1}{m}$, where the probability is over the
random draw of the permutation $\pi$. Hence,
\begin{align*}
    \sum_{t=1}^B &\Pr\bigl(D_t \,\wedge\, P_t < d/2\bigr)
\\ &\le
    \sum_{t=1}^B \Pr\bigl(D_t \mid P_t < d/2,\, \text{$s_t$ not paired with $r_t$}\bigr) + \sum_{t=1}^B \Pr\bigl(\text{$s_t$ paired with $r_t$}\bigr)
\\ &\le
    \sum_{t=1}^B \Pr\bigl(D_t \mid P_t < d/2,\, \text{$s_t$ not paired with $r_t$}\bigr) + \frac{B}{m}~.
\end{align*}
Let $\Pr' = \Pr\bigl(\,\cdot\mid P_t < d/2,\, \text{$s_t$ not paired with
$r_t$}\bigr)$. Note that the two points $s_t$ and $r_t$ have independent
block assignments when conditioned on $s_t$ not paired with $r_t$. Moreover,
conditioned on $P_t < d/2$, the event $D_t$ implies $s_t,r_t \in T_i$ for
some $i \in \neg L_t \cap \neg R_t$, where $|L_t|,|R_t| < \tfrac{d}{2}$ and
for any $S \ss \{1,\dots,d\}$ we use $\neg S$ to denote
$\{1,\dots,d\}\setminus S$.

Since, by definition of $L_t$, the block of $s_t$ is not in $L_t$, and there
were no previous queries involving $s_t$ and a point belonging to a block in
$\neg L_t$, we have that
\[
    \Pr'\bigl(s_t \in T_i \,\big|\, L_t \bigr) = \frac{1}{|\neg L_t|} \qquad \forall i \not\in L_t~.
\]
Likewise,
\[
    \Pr'\bigl(r_t \in T_j \,\big|\, R_t \bigr) = \frac{1}{|\neg R_t|} \qquad \forall j \not\in R_t~.
\]
Hence, for $L',R'$ ranging over all subsets of $\{1,\dots,d\}$ of size
strictly less than $\tfrac{d}{2}$,
\begin{align*}
    \Pr'(D_t)
&=
    \sum_{L',R'}\sum_{i \in \neg L' \cap \neg R'} \!\Pr'\bigl(s_t \in T_i \,\wedge\, r_t \in T_i \,\big|\, L_t = L',\, R_t = R'\bigr)\,\Pr'(L_t = L' \,\wedge\, R_t = R')
\\ &=
    \sum_{L',R'}\sum_{i \in \neg L' \cap \neg R'} \!\Pr'\bigl(s_t \in T_i \,\big|\, L_t = L'\bigr)\,\Pr'\bigl(r_t \in T_i \,\big|\, R_t = R'\bigr)\,\Pr'(L_t = L' \,\wedge\, R_t = R')~.
\\ &=
    \sum_{L',R'}\sum_{i \in \neg L' \cap \neg R'} \frac{1}{|\neg L'|}\,\frac{1}{|\neg R'|}\,\Pr'(L_t = L' \,\wedge\, R_t = R')
\\ &=
    \sum_{L',R'} \frac{|\neg L' \cap \neg R'|}{|\neg L'|\,|\neg R'|}\,\Pr'(L_t = L' \,\wedge\, R_t = R')
\le \frac{2}{d}
\end{align*}
because $|\neg L'| \ge \tfrac{d}{2}$, $|\neg R'| \ge \tfrac{d}{2}$ and $|\neg
L' \cap \neg R'| \le \min\{|\neg L'|,|\neg R'|\}$. Therefore, using $m \ge 2d$ we can write
\begin{align*}
    \sum_{t=1}^B \Pr\bigl(D_t \,\wedge\, P_t < d/2\bigr)
\le
    \frac{2B}{d} + \frac{B}{m} \le \frac{5B}{2d}
\end{align*}
where we used the lemma's assumption that $m\ge 2d$. Putting everything
together, we get the following upper bound on the expectation of
\eqref{eq:clubound}:
\begin{equation}\label{eq:clubound1}
\E\left[\sum_{i=1}^d \Ind{\neg E_i}\right] \leq \frac{5B}{2d}+\sqrt{2B}~.
\end{equation}
On the other hand, we have
\begin{equation}\label{eq:clubound2}
\E\left[\sum_{i=1}^d \Ind{\neg E_i}\right] = \E\left[\sum_{i=1}^{d}\left(1-\Ind{E_i}\right)\right]
= d-\sum_{i=1}^{d}\Pr(E_i)~.
\end{equation}
Combining \eqref{eq:clubound1} and \eqref{eq:clubound2}, we get that
\[
\sum_{i=1}^{d}\Pr(E_i) \geq d-\frac{5B}{2d}-\sqrt{2B}~.
\]
To finish the lemma's proof, suppose on the contrary that
$\sum_{i=1}^{d}\Pr(E_i)\leq \frac{d}{2}$. Then from the equation above, we
would get that
\[
\frac{d}{2} \geq d-\frac{5B}{2d}-\sqrt{2B}
\]
which implies $B\geq
\left(\frac{\sqrt{7}-\sqrt{2}}{5}\right)^2d^2> 0.06d^2$, contradicting the lemma's assumptions.
Therefore, we must have $\sum_{i=1}^{d}\Pr(E_i)> \frac{d}{2}$ as required.
\end{proof}
\begin{lemma}\label{lem:indep}
Suppose the kernel matrix $K$ is sampled according to the distribution $\sD$
as defined earlier (using a parameter $\bsigma\in \{0,1\}^d$). Let $A_i$ be
any event that, conditioned on $N_i$, depends only on $\beta_{i,1}$ and $\beta_{i,2}$,
(as returned by a deterministic algorithm based on
access to the kernel matrix), and let $g(N_i)$ be some non-negative function
of $N_i$. Then
\[
\E\Bigl[g(N_{i})\bigl(\Ind{A_i,\sigma_i=1}+\Ind{\neg A_i,\sigma_i=0}\bigr)\Bigr]
~\geq~
\frac{1}{2}\,\E\bigl[g(N_{i})\,\Pr(E_i \mid N_i)\bigr]~.
\]
\end{lemma}
\begin{proof}
We begin by noting that
\begin{align}
\E&\Bigl[g(N_{i})\bigl(\Ind{A_i,\sigma_i=1}+\Ind{\neg A_i,\sigma_i=0}\bigr)\Bigr]\notag\\
&=\E\Bigl[\E\bigl[g(N_{i})\left(\Ind{A_i,\sigma_i=1}+\Ind{\neg A_i,\sigma_i=0}\right) \mid N_{i}\bigr]\Bigr]\notag\\
&=
\E\Bigl[g(N_{i})\bigl(\Pr(A_i,\sigma_i=1 \mid N_{i}) + \Pr(\neg A_i,\sigma_i=0 \mid N_{i})\bigr)\Bigr]~.\label{eq:bgnot}
\end{align}
We now continue by analyzing the probabilities in the expression:
\begin{equation}\label{eq:preq0}
\Pr(A_i,\sigma_i=1 \mid N_i)+\Pr(\neg A_i,\sigma_i=0 \mid N_i)~.
\end{equation}
We will need two auxiliary results. First, we argue that for all $i$,
\begin{equation}\label{eq:preq1}
\Pr(E_i \mid \sigma_i=0,N_i)=\Pr(E_i \mid \sigma_i=1,N_i)=\Pr(E_i \mid N_i)~.
\end{equation}
To prove this, we note that since the algorithm is
  deterministic, the occurrence of the event $E_i$ is determined by the kernel matrix, and more
  specifically the entries of the kernel matrix observed by the algorithm.
  Therefore, if the kernel matrix is such that $E_i$ occurs, then the algorithm's
  output would not change if we flip the value of $\sigma_i$ as it only
  affects entries which were not touched by the algorithm. Therefore,
  $\Pr(E_i \mid \sigma_i=0,N_i)=\Pr(E_i \mid \sigma_i=1,N_i)$. Since $\sigma_i$ is either $0$
  or $1$, this means that these probabilities also equal $\Pr(E_i \mid N_i)$.

Second, we argue that
\begin{equation}\label{eq:preq2}
\Pr(A_i \mid E_i,\sigma_i=0,N_i)=\Pr(A_i \mid E_i,\sigma_i=1,N_i)~.
\end{equation}
This holds because if $E_i$ occurs, then $\beta_{i,1},\beta_{i,2}$ depend
only on entries which are independent of $\sigma_i$. Moreover, $N_i$ is also
independent of $\sigma_i$. Therefore $A_i$, which is assumed to depend only
on $\beta_{i,1},\beta_{i,2}$ when conditioned on $N_i$, is also independent of
$\sigma_i$ when conditioned on $E_i$ and $N_i$, from which \eqref{eq:preq2} follows.

Using \eqref{eq:preq1}, \eqref{eq:preq2}, and the fact that $\sigma_i$ is
uniformly drawn from $\{0,1\}$ and independent of $N_i$, we have that
\eqref{eq:preq0} equals
\begin{align*}
&\Pr(A_i,\sigma_i=1 \mid N_i)
+\Pr(\neg A_i,\sigma_i=0 \mid N_i)\\
&=\Pr(\sigma_i=1)\Pr(A_i\mid\sigma_i=1,N_i)+\Pr(\sigma_i=0)\Pr(\neg A_i\mid\sigma_i=0,N_i)\\
&=\frac{1}{2}\Bigl(\Pr(A_i \mid \sigma_i=1,N_i)
+\Pr(\neg A_i \mid \sigma_i=0,N_i)\Bigr)\\
&=\frac{1}{2}\Bigl(1-\Pr(A_i \mid \sigma_i=0,N_i)
+\Pr(A_i \mid \sigma_i=1,N_i)\Bigr)\\
&=\frac{1}{2}\Bigl(1-\Pr(E_i\mid\sigma_i=0,N_i)\Pr(A_i\mid E_i,\sigma_i=0,N_i)
-\Pr(\neg E_i\mid\sigma_i=0,N_i)\Pr(A_i\mid\neg E_i,\sigma_i=0,N_i)\Bigr.\\
&\Bigl.\qquad+\;\Pr(E_i\mid\sigma_i=1,N_i)\Pr(A_i\mid E_i,\sigma_i=1,N_i)
+\Pr(\neg E_i\mid\sigma_i=1,N_i)\Pr(A_i\mid\neg E_i,\sigma_i=1,N_i)\Bigr)\\
&=\frac{1}{2}\Bigl(1+\Pr(E_i\mid N_i)\bigl(\Pr(A_i \mid E_i,\sigma_i=1,N_i)-\Pr(A_i \mid E_i,\sigma_i=0,N_i)\bigr)\Bigr.\\
&\Bigl.\qquad+\;\Pr(\neg E_i \mid N_i)\bigl(\Pr(A_i \mid \neg E_i,\sigma_i=1,N_i)-\Pr(A_i \mid \neg E_i,\sigma_i=0,N_i)\bigr)\Bigr)\\
&\geq \frac{1}{2}\Bigl(1+\Pr(E_i \mid N_i)\times 0+\Pr(\neg E_i\mid N_i)\times(-1)\Bigr) ~=~ \frac{1}{2}\bigl(1-\Pr(\neg E_i\mid N_i)\bigr)~=~\frac{1}{2}\Pr(E_i \mid N_i)~.
\end{align*}
Plugging this lower bound on \eqref{eq:preq0} back into \eqref{eq:bgnot}, the
result follows.
\end{proof}

Finally, we will also require the following auxiliary result:
\begin{lemma}\label{lem:getrid}
Let $N=\sum_{i=1}^{m/2} 2X_i$ where $X_1,\dots,X_{m/2}$ are i.i.d.\ Bernoulli random variables
with parameter $1/d$. Also, let $g(N)$ be a non-negative function of $N$.
Then for any event $Z$,
\[
\E\bigl[g(N)\,\Pr(Z\mid N=n)\bigr]
\geq \left(\min_{\frac{m}{2d}\leq n\leq \frac{2m}{d}}g(n)\right)\left(\Pr(Z)-2\exp\left(-\frac{m}{8d}\right)\right).
\]
\end{lemma}
\begin{proof}
  Let $S=\left\{\frac{m}{2d},\frac{m}{2d}+1,\ldots,\frac{2m}{d}\right\}$. We can lower bound the expectation by
\begin{align}
  &\sum_{n\in S}\Pr(N=n)g(n)\Pr(Z \mid N=n)
  ~\geq~ \left(\min_{n\in S}g(n)\right)\sum_{n\in S}\Pr(N=n)\Pr(Z \mid n)\notag\\
  &=\left(\min_{n\in S}g(n)\right)\left(\Pr(Z)-\sum_{n\notin S}\Pr(N=n)\Pr(Z \mid n)\right)\notag\\
  &\geq \left(\min_{n\in S}g(n)\right)\left(\Pr(Z)-\sum_{n\notin S}\Pr(N=n)\right)\notag\\
  &= \left(\min_{n\in S}g(n)\right)\Bigl(\Pr(Z)-\Pr(N\notin S)\Bigr)~.\label{eq:asdf}
\end{align}
Since $N$ is distributed as twice the sum of $m/2$ i.i.d.\ Bernoulli random
variables with parameter $1/d$, so its expectation is $m/d$, and by multiplicative Chernoff bounds
and union bounds,
\[
\Pr(N\notin S) = \Pr\left(N> \frac{2m}{d}\right)+\Pr\left(N< \frac{m}{2d}\right)
\leq \exp\left(-\frac{m}{3d}\right)+\exp\left(-\frac{m}{8d}\right)\leq 2\exp\left(-\frac{m}{8d}\right)~.
\]
Substituting this back into \eqref{eq:asdf}, the result follows.
\end{proof}

\subsubsection{Proof of \thmref{thm:absmain}}

Suppose we pick $y_t=\frac{1}{\sqrt{d}}$ for all $t$. Using Yao's minimax
principle, it is sufficient to prove a lower bound for
\begin{equation}\label{eq:expabs}
\E\left[\frac{1}{m}\sum_{t=1}^{m}
\left|\balpha^\top K \be_t-\frac{1}{\sqrt{d}}\right|
-\min_{\balpha\,:\,\balpha^\top K \balpha\leq 2}\frac{1}{m}\sum_{t=1}^{m}
\left|\balpha^\top K \be_t-\frac{1}{\sqrt{d}}\right|\right]
\geq \frac{1}{70\sqrt{d}}\;\;,
\end{equation}
where the expectation is with respect to the kernel matrix $K$ drawn according to the distribution
$\sD$ specified earlier, and $\balpha$ is any deterministic function of $K$ encoding the learning algorithm. This ensures that for any (possibly randomized) algorithm, there exists some $K$
which satisfies the theorem statement.

First, we will show that for any $K\in \sK_{2d,m}$, there exists some
$\balpha$ such that
\begin{equation}\label{eq:toshow}
\frac{1}{m}\sum_{t=1}^{m}
\left|\balpha^\top K \be_t-\frac{1}{\sqrt{d}}\right|=0
\qquad\mbox{ and }\qquad
\balpha^\top K \balpha\leq 2~.
\end{equation}
This implies that \eqref{eq:expabs} can be re-written as
\begin{equation}\label{eq:expabs2}
\E\left[\frac{1}{m}\sum_{t=1}^{m}
\left|\balpha^\top K \be_t-\frac{1}{\sqrt{d}}\right|
\right]~.
\end{equation}
To see this, we utilize \lemref{lem:beta} to rewrite \eqref{eq:toshow} as
\begin{align*}
\frac{1}{m}\sum_{t=1}^{m}\ell\left(\balpha^\top K \be_i,y\right) = \sum_{i=1}^{d}&\frac{N_i}{2m}\left(\left|\beta_{i,1}+\sigma_i \beta_{i,2}-\frac{1}{\sqrt{d}}\right|
+\left|\sigma_i\beta_{i,1}+\beta_{i,2}-\frac{1}{\sqrt{d}}\right|\right)=0, \mbox{ and}
\\
\balpha^\top K \balpha = \sum_{i=1}^{d}&\left(\beta_{i,1}^2+\beta_{i,2}^2+2\sigma_i \beta_{i,1}\beta_{i,2}\right)\leq 2\;\;,
\end{align*}
where $\{\beta_{i,1},\beta_{i,2}\}$ are the appropriate functions of
$\balpha$. Note that these constraints are indeed satisfied for any
$\balpha$ for which $\beta_{i,1}=\beta_{i,2}=\frac{1}{\sqrt{d}}$ if
$\sigma_i=0$, and $\beta_{i,1}=\frac{1}{\sqrt{d}},\beta_{i,2}=0$ if
$\sigma_i=1$.
Again using \lemref{lem:beta}, we can rewrite
\eqref{eq:expabs2} as
\begin{equation}\label{eq:expabs3}
\E\left[\sum_{i=1}^{d}\frac{N_i}{2m}\left(\left|\beta_{i,1}+\sigma_i \beta_{i,2}-\frac{1}{\sqrt{d}}\right|
+\left|\sigma_i\beta_{i,1}+\beta_{i,2}-\frac{1}{\sqrt{d}}\right|\right)\right]~.
\end{equation}
Let us consider the expression $\left|\beta_{i,1}+\sigma_i
\beta_{i,2}-\frac{1}{\sqrt{d}}\right|
+\left|\sigma_i\beta_{i,1}+\beta_{i,2}-\frac{1}{\sqrt{d}}\right|$ for some
fixed choice of the kernel matrix $K$. In particular:
\begin{itemize}
  \item If $\sigma_i=1$ and $\beta_{i,1}+\beta_{i,2}\geq
      \frac{3}{2\sqrt{d}}$, then
      \[
      \left|\beta_{i,1}+\sigma_i
\beta_{i,2}-\frac{1}{\sqrt{d}}\right|
+\left|\sigma_i\beta_{i,1}+\beta_{i,2}-\frac{1}{\sqrt{d}}\right| ~=~ 2\left|\beta_{i,1}+\beta_{i,2}-\frac{1}{\sqrt{d}}\right| ~\geq~ \frac{1}{\sqrt{d}}~.
      \]
  \item If $\sigma_i=0$ and $\beta_{i,1}+\beta_{i,2}< \frac{3}{2\sqrt{d}}$,
      then either $\beta_{i,1}$ or $\beta_{i,2}$ must be less than
      $\frac{3}{4\sqrt{d}}$, and therefore
      \begin{align*}
      \left|\beta_{i,1}+\sigma_i
\beta_{i,2}-\frac{1}{\sqrt{d}}\right|
+\left|\sigma_i\beta_{i,1}+\beta_{i,2}-\frac{1}{\sqrt{d}}\right| &=
\left|\beta_{i,1}-\frac{1}{\sqrt{d}}\right|
+\left|\beta_{i,2}-\frac{1}{\sqrt{d}}\right|\\
&\geq \left|\frac{3}{4\sqrt{d}}-\frac{1}{\sqrt{d}}\right| =
\frac{1}{4\sqrt{d}}~.
      \end{align*}
\end{itemize}
Let $A_i$ be the event that $\beta_{i,1}+\beta_{i,2} \geq
      \frac{3}{2\sqrt{d}}$. Since the algorithm is deterministic,
      $\{\beta_{i,1},\beta_{i,2}\}$ and hence $A_i$ is determined by the
      kernel matrix $K$. By the analysis above, we can lower bound \eqref{eq:expabs3} by
\begin{align*}
\E&\left[\sum_{i=1}^{d}\frac{N_i}{2m}\left(\Ind{A_i,\sigma_i=1}\frac{1}{\sqrt{d}}+\Ind{\neg A_i,\sigma_i=0}\frac{1}{4\sqrt{d}}\right)\right]\\
&\geq\frac{1}{8m\sqrt{d}}\sum_{i=1}^{d}\E\Bigl[N_i\bigl(\Ind{A_i,\sigma_i=1}+\Ind{\neg A_i,\sigma_i=0}\bigr)\Bigr]~.\\
\end{align*}
By \lemref{lem:indep} and \lemref{lem:getrid} this is lower bounded by
\[
\frac{1}{16m\sqrt{d}}\sum_{i=1}^{d}\E\bigl[N_i\,\Pr(E_i \mid N_i)\bigr]
~\geq~ \frac{1}{32\sqrt{d}}\left(\frac{1}{d}\sum_{i=1}^{d}\Pr(E_i)-2\exp\left(-\frac{m}{8d}\right)\right)~.
\]
Since we assumed that $B< \frac{3}{50}d^2$ and $m\geq 2^7d$, we can
apply \lemref{lem:block}, which lower bounded this by
\[
\frac{1}{32\sqrt{d}}\left(\frac{1}{2} -2\exp\left(-\frac{m}{8d}\right)\right)
~\geq~ \frac{1}{70\sqrt{d}}~,
\]
where we used again the assumption that $m\geq 2^7d$.
\hfill $\qed$

\subsubsection{Proof of \thmref{thm:lossmain}}
\label{subsec:general}

The proof is broadly similar to the one of \thmref{thm:absmain}, but using a
generic loss rather than the absolute loss.

Suppose we pick $y_t=y\in\sY$ for all $t$, where $y\in \sY$ will be determined later.
Using Yao's minimax principle, it is sufficient to prove a lower bound for
\begin{equation}
\E_{K}\left[\left(\frac{1}{m}\sum_{t=1}^{m}\ell\left(\balpha^\top K \be_t,y\right)+\frac{\lambda}{2}\balpha^\top K \balpha\right)
-\min_{\balpha}\left(\frac{1}{m}\sum_{t=1}^{m}\ell\left(\balpha^\top K \be_t,y\right)+\frac{\lambda}{2}\balpha^\top K \balpha\right)\right]
\label{eq:exploss}
\end{equation}
where ---as in the proof of \thmref{thm:absmain}--- the expectation is with respect to the kernel matrix $K$ drawn according to the distribution $\sD$, and $\balpha$ is any deterministic function of $K$ encoding the learning algorithm. This ensures that for any (possibly randomized) algorithm, there exists some $K$ which satisfies the theorem statement.

Utilizing \lemref{lem:beta}, we can rewrite (\ref{eq:exploss}) as
\begin{align*}
\E&\left[\left(\sum_{i=1}^{d}\frac{N_i}{2m}\Bigl(\ell(\beta_{i,1}+\sigma_i \beta_{i,2},y)
+\ell(\sigma_i \beta_{i,1}+\beta_{i,2},y)\Bigr)+\frac{\lambda}{2}\sum_{i=1}^{d}\left(\beta_{i,1}^2+\beta_{i,2}^2+2\sigma_i \beta_{i,1}\beta_{i,2}\right)\right)\right.\\
&-\min_{\{\beta_{i,1},\beta_{i,2}\}}\left(
\sum_{i=1}^{d}\frac{N_i}{2m}\Bigl(\ell(\beta_{i,1}+\sigma_i \beta_{i,2},y)
+\ell(\sigma_i \beta_{i,1}+\beta_{i,2},y)\Bigr) \right.\\
&\left.\left. \quad+\,\frac{\lambda}{2}\sum_{i=1}^{d}\left(\beta_{i,1}^2+\beta_{i,2}^2+2\sigma_i \beta_{i,1}\beta_{i,2}\right)\right)\right]\\
&=\sum_{i=1}^{d}\E\left[
\frac{N_i}{2m}\Bigl(\ell(\beta_{i,1}+\sigma_i \beta_{i,2},y)
+\ell(\sigma_i \beta_{i,1}+\beta_{i,2},y)\Bigr)+\frac{\lambda}{2}\left(\beta_{i,1}^2+\beta_{i,2}^2+2\sigma_i \beta_{i,1}\beta_{i,2}\right)\right.\\
&\left.\quad-\min_{\beta_{i,1},\beta_{i,2}}\left(
\frac{N_i}{2m}\Bigl(\ell(\beta_{i,1}+\sigma_i \beta_{i,2},y)
+\ell(\sigma_i \beta_{i,1}+\beta_{i,2},y)\Bigr)+\frac{\lambda}{2}\left(\beta_{i,1}^2+\beta_{i,2}^2+2\sigma_i \beta_{i,1}\beta_{i,2}\right)
\right)\right]~.
\end{align*}
This can be written in a simplified form as
\begin{equation}\label{eq:simplified1}
\sum_{i=1}^{d}\E\bigl[g^{\sigma_i}_{N_i}(\beta_{i,1},\beta_{i,2})\bigr]\;\;,
\end{equation}
where
\begin{align*}
g^\sigma_n(u,v)&=f^\sigma_n(u,v)-\min_{u,v}f^\sigma_n(u,v), \mbox{  and} \\
f^\sigma_n(u,v)&=\frac{n}{2m}\bigl(\ell(u+\sigma v,y)
+\ell(\sigma u+v,y)\bigr)+\frac{\lambda}{2}\left(u^2+v^2+2\sigma u v\right)~.
\end{align*}
Now, let $A_i$ be the event that
$g^0_{N_i}(\beta_{i,1},\beta_{i,2})<g^1_{N_i}(\beta_{i,1},\beta_{i,2})$. We
consider two cases:
\begin{itemize}
  \item If $\sigma_i=1$ and $A_i$ occurs, then
\[
    g^{\sigma_i}_{N_i}(\beta_{i,1},\beta_{i,2})= \max_{\sigma}
    g^{\sigma}_{N_i}(\beta_{i,1},\beta_{i,2})\geq
    \min_{u,v}\max_{\sigma}g^{\sigma}_{N_i}(u,v)~.
\]
  \item If $\sigma_i=0$ and $A_i$ does not occur, then
\[
    g^{\sigma_i}_{N_i}(\beta_{i,1},\beta_{i,2})= \max_{\sigma}
    g^{\sigma}_{N_i}(\beta_{i,1},\beta_{i,2})\geq
    \min_{u,v}\max_{\sigma}g^{\sigma}_{N_i}(u,v)~.
\]
\end{itemize}
Therefore, using the fact that $g^\sigma_n$ is non-negative by definition, we
have
\[
g^{\sigma_i}_{N_i}(\beta_{i,1},\beta_{i,2}) \geq \bigl(\Ind{A_i,\sigma_i=1}+\Ind{\neg A_i,\sigma_i=0}\bigr)\min_{u,v}\max_{\sigma}g^\sigma_{N_i}(u,v)~.
\]
Substituting this back into \eqref{eq:simplified1}, and using \lemref{lem:indep},
\lemref{lem:getrid} and \lemref{lem:block} in order, we get a lower bound of
the form
\begin{align}
\sum_{i=1}^{d}&\E\left[\left(\min_{u,v}\max_{\sigma}g^\sigma_{N_i}(u,v)\right)\bigl(\Ind{A_i,\sigma_i=1}+\Ind{\neg A_i,\sigma_i=0}\bigr)\right]\notag\\
&\geq \frac{1}{2}\sum_{i=1}^{d}\E\left[\left(\min_{u,v}\max_{\sigma}g^\sigma_{N_i}(u,v)\right)\Pr(E_i \mid N_i)\right]\notag\\
&\geq \frac{1}{2}\sum_{i=1}^{d}\left(\min_{\frac{m}{2d}\leq {n_i} \leq \frac{2m}{d}}\min_{u,v}\max_{\sigma}g^\sigma_{n_i}(u,v)\right)\left(\Pr(E_i)-2\exp\left(-\frac{m}{8d}\right)\right)\notag\\
&= \frac{1}{2}\left(\min_{\frac{m}{2d}\leq {n} \leq \frac{2m}{d}}\min_{u,v}\max_{\sigma}g^\sigma_{n}(u,v)\right)\sum_{i=1}^{d}\left(\Pr(E_i)-2\exp\left(-\frac{m}{8d}\right)\right)\notag\\
&\geq \frac{d}{2}\left(\min_{\frac{m}{2d}\leq {n} \leq
\frac{2m}{d}}\min_{u,v}\max_{\sigma}g^\sigma_{n}(u,v)\right)\left(\frac{1}{2}-2\exp\left(-\frac{m}{8d}\right)\right)\notag\\
&= d \left(\frac{1}{4}-\exp\left(-\frac{m}{8d}\right)\right)\left(\min_{\frac{m}{2d}\leq {n} \leq
\frac{2m}{d}}\min_{u,v}\max_{\sigma}g^\sigma_{n}(u,v)\right)\notag
\end{align}
where the first inequality follows from Lemma~\ref{lem:indep}, the second inequality is from Lemma~\ref{lem:getrid}, and the third inequality
is by Lemma~\ref{lem:block}.
Since we assume $m\geq 2^{7}d$, this is at least
\begin{equation}
\frac{d}{5}\left(\min_{\frac{m}{2d}\leq {n} \leq
\frac{2m}{d}}\min_{u,v}\max_{\sigma}g^\sigma_{n}(u,v)\right)~.\label{eq:minmax}
\end{equation}
We now turn to analyze ${\displaystyle \min_{u,v}\max_{\sigma}g^\sigma_{n}(u,v)}$. By
definition of $g^\sigma_{n}(u,v)$, we have that
\[
g^0_n(u,v) = f^0_n(u,v)-\min_{u,v}f^0_n(u,v) = \frac{n}{2m}\bigl(\ell(u,y)+\ell(v,y)\bigr)+\frac{\lambda}{2}(u^2+v^2)-\min_{u,v}f^0_n(u,v)~.
\]
It is readily seen that this function is $\lambda$-strongly convex in
$(u,v)$, and attains a minimal value of $0$ at some $(u_1^*,u_1^*)$, where
\[
u^*_1 = \argmin_{u} \frac{n}{m}\ell(u,y)+\lambda u^2~.
\]
Using the property of $\lambda$-strong convexity, we have for all $u,v$ that
\begin{equation}\label{eq:f0}
g^0_{n}(u,v) = g^0_n(u,v)-g^0_n(u_1^*,u_1^*) \geq \frac{\lambda}{2}\norm{(u,v)-(u_1^*,u_1^*)}^2 =
\frac{\lambda}{2}\bigl((u-u_1^*)^2+(v-u_1^*)^2\bigr)~.
\end{equation}
Also, by definition,
\[
g^1_n(u,v) = f^1_n(u,v)-\min_{u,v}f^1_n(u,v) = \frac{n}{m}\ell(u+v,y)+\frac{\lambda}{2}(u+v)^2-\min_{u,v}f^1_n(u,v)
\]
which is a $\lambda$ strongly-convex function in $u+v$, and attains a minimal
value of $0$ at any $u_2,v_2$ such that $u_2^*=u_2+v_2$, where
\[
u_2^* ~=~ \argmin_{u}\frac{n}{m}\ell(u,y)+\frac{\lambda}{2} u^2.
\]
Using the property of $\lambda$-strong convexity, we have for all $u,v$ that
\begin{equation}\label{eq:f1}
g^1_n(u,v) = g^1(u,v)-g^1(u_2,v_2) \geq \frac{\lambda}{2}\bigl((u+v)-(u_2+v_2)\bigr)^2 = \frac{\lambda}{2}(u+v-u_2^*)^2~.
\end{equation}
Combining \eqref{eq:f0} and \eqref{eq:f1}, and using the fact that the
maximum is lower bounded by the average, this implies that
\begin{align*}
\min_{u,v}\max_{\sigma}g^\sigma_{n}(u,v) &\geq
\min_{u,v}\max\left\{\frac{\lambda}{2}\bigl((u-u_1^*)^2+(v-u_1^*)^2\bigr),\,\frac{\lambda}{2}(u+v-u_2^*)^2\right\}\\
&\geq \min_{u,v}\frac{\lambda}{4}\Bigl((u-u_1^*)^2+(v-u_1^*)^2+(u+v-u_2^*)^2\Bigr)~.
\end{align*}
A straightforward calculation reveals that this expression is minimized at $u
= v=\frac{1}{3}(u_1^*+u_2^*)$, leading to a value of
$\frac{\lambda}{12}(2u_1^*-u_2^*)^2$. To summarize, we showed that
\[
\min_{u,v}\max_{\sigma}g^\sigma_{n}(u,v) \geq \frac{\lambda}{12}(2u_1^*-u_2^*)^2
\]
where
\[
u^*_1 = \arg\min_{u} \frac{n}{m}\ell(u,y)+\lambda u^2
\qquad\text{and}\qquad
u_2^* ~=~ \arg\min_{u}\frac{n}{m}\ell(u,y)+\frac{\lambda}{2} u^2~.
\]
This computation holds for any value of $y$, and therefore we have
\[
\min_{u,v}\max_{\sigma}g^\sigma_{n}(u,v) \geq \max_{y\in\sY}\frac{\lambda}{12}(2u_1^*-u_2^*)^2
\]
where $u_1^*,u_2^*$ are as defined above. Substituting this back into
\eqref{eq:minmax}, we get
\begin{equation}\label{eq:f2}
\frac{1}{60}\lambda d\left(\min_{\frac{m}{2d}\leq
{n} \leq \frac{2m}{d}}\max_{y\in\sY}\,(2u_1^*-u_2^*)^2\right)~.
\end{equation}
Finally, to write this in a simpler form, let $p=\frac{m}{nd}$. Then we can
equivalently write $u_1^*,u_2^*$ as
\begin{align*}
u^*_1 &= \argmin_{u} \ell(u,y)+\frac{m}{n}\lambda u^2 = \argmin_{u} \ell(u,y)+p\lambda d u^2\\
u_2^* &= \argmin_{u} \ell(u,y)+\frac{m\lambda}{2n} u^2 = \argmin_{u} \ell(u,y)+\frac{p\lambda d}{2} u^2~.
\end{align*}
Moreover, the constraint $\frac{m}{2d}\leq {n} \leq \frac{2m}{d}$ implies
that $p\in \left[\frac{1}{2},2\right]$, so we can lower bound \eqref{eq:f2}
by
\[
\frac{1}{60}\lambda d\left(\min_{p\in \left[\frac{1}{2},2\right]}\max_{y\in\sY}\,(2u_1^*-u_2^*)^2\right)
\]
as desired.
\hfill$\qed$

\subsection{Proof of \thmref{thm:lowrank}}
\label{subsec:theorem3}

Recall that we assume that $m$ is divisible by $2d$.
  Given $t=1,\ldots,m$, define
  \[
  i(t) = 1+\left\lfloor \frac{t-1}{m/2d}\right\rfloor
  \]
  to be the partition function of $\{1,\ldots,m\}$ into $2d$ equal-sized
  blocks:
  \begin{align*}
  &i(1)=i(2)=\cdots=i(m/2d)=1
  \\
  &i(m/2d+1)=\cdots=i(2m/2d)=2
  \end{align*}
  and so on, until $i(m)=2d$.

  Suppose we choose the target values $y_1,\ldots,y_m$ according to
  $y_i=z_{i(t)}$, where $\bz=(z_1,\ldots,z_{2d})$ is to be chosen later, and let
  $\bx_t=\bv_{i(t)}$ for $t=1,\dots,m$.
  Recall that  $k(\bv_i,\bv_j)=\Ind{\bv_i=\bv_j}$.
  It is easily seen that these instances induce a block-diagonal kernel matrix
  $K\in \sK_{2d,m}$, composed of $2d$ all-one blocks of equal size $m/(2d)$.
  Moreover, any low-rank matrix $K'$ used by the algorithm will also have a
  block-wise structure (with possibly different values for the entries), where
  $K'_{t,t'}=\inner{\phi(\bv_{i(t)}),\phi(\bv_{i(t')})}$.

  Given any such block-wise kernel matrix $K$, composed of $2d$ uniform blocks of size $m/2d$, let
  $G^{K}$ be the $2d\times 2d$ matrix defined as
  $G_{i(t),i(t')}=K_{t,t'}$. Note that since $K$ is symmetric,
  $G^{K}$ is symmetric as well.
  Finally, given some coefficient vector $\balpha$, define $\bbeta$ as
  \[
  \forall i=1,\ldots,d,~~\beta_i = \sum_{t\,:\,i(t)=i}\alpha_t~.
  \]
  With this notation, we can re-write the objective function and resulting solution using the following lemma.
  \begin{lemma}\label{lem:alphabeta}
      For any block matrix $K$, where $K_{t,t'}=K_{r,r'}$ if $i(t)=i(r)$ and $i(t')=i(r')$, and any coefficient vector $\balpha$ with corresponding
      $\bbeta$,
      we have
      \begin{align*}
      \frac{1}{m}\sum_{t=1}^{m}\left(\balpha^\top K\be_t-y_t\right)^2+\frac{\lambda}{2} \balpha^\top K \balpha
      &=
      \frac{1}{m}\left(\balpha^\top\left(K+\frac{m\lambda}{2}I\right)K\balpha-2\by^\top K\balpha+\norm{\by}^2\right)\\
      &=
      \frac{1}{2d}\left(\bbeta^\top \left(G^{K}+d\lambda I\right)G^{K}\bbeta-2\bz^\top G^{K} \bbeta+\norm{\bz}^2\right).
      \end{align*}
      Moreover, if $\balpha= \left(K+\frac{\lambda
      m}{2}I\right)^{-1}\by$,
      then $\bbeta= (G^{K}+d\lambda I)^{-1}\bz$.
  \end{lemma}
  The proof is a technical exercise, and appears separately in \subsecref{subsec:techlemma}.

  In our case, we chose the training instances so that $K$ is a block-diagonal matrix
  composed of $2d$ equal-sized all-ones block. Therefore, $G^K$ in our case is simply the
  $d\times d$ identity matrix. By \lemref{lem:alphabeta}, we can write the objective function as
  \[
  \frac{1}{m}\sum_{t=1}^{m}\left(\balpha^\top K\be_t-y_t\right)^2+\frac{\lambda}{2} \balpha^\top K\balpha
  ~=~\frac{1}{2d}\left((1+d\lambda)\norm{\bbeta}^2-2\bz^\top\bbeta+\norm{\bz}^2\right).
  \]
  This function is $\frac{1+d\lambda}{d}$-strongly convex in $\bbeta$, and is
  minimized at $\bbeta^*=\frac{1}{1+d\lambda}\bz$. Therefore, the error
  obtained by any other solution $\bbeta$ is at least
  \begin{equation}\label{eq:excbeta}
    \frac{1+d\lambda}{2d}\norml{\bbeta-\frac{1}{1+d\lambda}\bz}^2.
  \end{equation}

  According to \lemref{lem:alphabeta} and the definition of the algorithm, the $\bbeta$ corresponding to the coefficient vector $\balpha$
  returned by the learning algorithm (using a kernel matrix $K'$) satisfies $\bbeta ~=~ (G^{K'}+d\lambda
  I)^{-1}\bz$. Plugging this back into \eqref{eq:excbeta}, we get an error lower bound of
  \begin{equation}\label{eq:lowrank2}
    \frac{1+d\lambda}{2d}\norml{(G^{K'}+d\lambda I)^{-1}\bz-\frac{1}{1+d\lambda}\bz}^2
    ~=~\frac{1+d\lambda}{2d}\norml{\left((G^{K'}+d\lambda I)^{-1}-\frac{1}{1+d\lambda}I\right)\bz}^2.
  \end{equation}
  Let $USU^\top$ be the spectral decomposition of $G^{K'}$, where $U = \bigl[\bu_1,\dots,\bu_{2d}\bigr] \in \reals^{2d\times 2d}$ is an orthonormal matrix, and $S$ is a diagonal matrix
  with eigenvalues $s_1\leq s_2\leq\ldots s_{2d}$ on the diagonal. Moreover, since $K'$ is a matrix of rank at most $d$, it follows
  that $G^{K'}$ is also of rank at most $d$, hence $s_1=\cdots=s_d=0$. We can
  therefore re-write \eqref{eq:lowrank2} as
  \begin{align*}
  \frac{1+d\lambda}{2d}&\norml{\left(U(S+d\lambda I)^{-1}U^\top-\frac{1}{1+d\lambda}I\right)\bz}^2\\
  &=~  \frac{1+d\lambda}{2d}\norml{U\left((S+d\lambda I)^{-1}-\frac{1}{1+d\lambda}I\right)U^\top\bz}^2\\
  &=~ \frac{1+d\lambda}{2d}\norml{\left((S+d\lambda I)^{-1}-\frac{1}{1+d\lambda}I\right)U^\top\bz}^2\\
  &=~ \frac{1+d\lambda}{2d}\sum_{i=1}^{2d}\left(\left(\frac{1}{s_i+d\lambda}-\frac{1}{1+d\lambda}\right)\bu_i^\top \bz\right)^2\\
  &\geq~ \frac{1+d\lambda}{2d}\sum_{i=1}^{d}\left(\left(\frac{1}{s_i+d\lambda}-\frac{1}{1+d\lambda}\right)\bu_i^\top \bz\right)^2\\
  &=~ \frac{1+d\lambda}{2d}\sum_{i=1}^{d}\left(\left(\frac{1}{d\lambda}-\frac{1}{1+d\lambda}\right)\bu_i^\top \bz\right)^2\\
  &=~ \frac{1+d\lambda}{2d}\frac{1}{\bigl(d\lambda(1+d\lambda)\bigr)^2}\sum_{i=1}^{d}\left(\bu_i^\top \bz\right)^2\\
  &=~ \frac{1}{2(d\lambda)^2(1+d\lambda)d}\sum_{i=1}^{d}\left(\bu_i^\top \bz\right)^2~.\\
  \end{align*}
  We are now free to choose $\bz = (z_1,\dots,z_{2d})$, which induces some choice of the target
  values $y_1,\ldots,y_m$, to get the final bound. In
  particular, we argue that there exist some $\bz\in \{-1,+1\}^{2d}$ such
  that
  \[
  \sum_{i=1}^{d}\left(\bu_i^\top \bz\right)^2 \geq d
  \]
  from which the result follows. To show this, we use the probabilistic
  method: Suppose that $\bz$ is chosen uniformly at random from
  $\{-1,+1\}^{2d}$. Then
  \begin{align*}
  \E\left[\sum_{i=1}^{d}\left(\bu_i^\top\bz\right)^2\right] &=~
  \E\left[\sum_{i=1}^{d}\bu_i^\top \bz \bz^\top \bu_i\right]
  ~=~ \sum_{i=1}^{d}\bu_i^\top \E[\bz \bz^\top]\bu_i\\
  &=~ \sum_{i=1}^{d}\bu_i^\top I \bu_i ~=~ \sum_{i=1}^{d}\norm{\bu_i}^2 = d~.
  \end{align*}
  This means that there must exist some $\bz\in \{-1,+1\}^{2d}$ such that
  $\sum_{i=1}^{d}\left(\bu_i^\top\bz\right)^2\geq d$ as required.
\hfill$\qed$

\subsection{Proof of \lemref{lem:alphabeta}}\label{subsec:techlemma}
    The fact that
    \[
    \frac{1}{m}\sum_{t=1}^{m}\left(\balpha^\top K\be_t-y_t\right)^2+\frac{\lambda}{2} \balpha^\top K \balpha
    ~=~
    \frac{1}{m}\left(\balpha^\top\left(K+\frac{m\lambda}{2}I\right)K\balpha-2\by^\top K\balpha+\norm{\by}^2\right)
    \]
    is a straightforward exercise. The second part of the equation can be
    shown as follows:
    \begin{align}
      &\frac{1}{m}\sum_{t=1}^{m}\left(\balpha^\top K\be_t-y_t\right)^2+\frac{\lambda}{2} \balpha^\top K \balpha\notag\\
      &=\frac{1}{m}\sum_{j=1}^{2d}\sum_{t\,:\,i(t)=j}\left(\sum_{j'=1}^{2d}\>\sum_{t':\,i(t')=j'}\alpha_{t'}K_{t',t}-y_t\right)^2
      +\frac{\lambda}{2}\sum_{j,j'=1}^{2d}\>\sum_{t\,:\,i(t)=j}\>\sum_{t':\,i(t')=j'}\alpha_t K_{t,t'}\alpha_{t'}\notag\\
      &=\frac{1}{m}\sum_{j=1}^{2d}\sum_{t\,:\,i(t)=j}\left(\sum_{j'=1}^{2d}\beta_j G^{K}_{j',j}-z_j\right)^2+\frac{\lambda}{2}\sum_{j,j'=1}^{2d}\beta_j G^{K}_{j,j'}\beta_{j'}\notag\\
      &=\frac{1}{m}\sum_{j=1}^{2d}\frac{m}{2d}\left(\bbeta^\top G^{K} \be_j-z_j\right)^2+\frac{\lambda}{2}\bbeta^\top G^{K}\bbeta\notag\\
      &=\frac{1}{2d}\norml{G^{K}\bbeta-\bz}^2+\frac{\lambda}{2}\bbeta^\top G^{K}\bbeta\notag\\
      &=\frac{1}{2d}\left(\bbeta^\top G^{K}G^{K}\bbeta-2\bz^\top G^{K}\bbeta+\norm{\bz}^2\right)
      +\frac{\lambda}{2}\bbeta^\top G^{K}\bbeta\notag\\
      &= \frac{1}{2d}\left(\bbeta^\top \left(G^{K}+d\lambda I\right)G^{K}\bbeta-2\bz^\top G^{K} \bbeta+\norm{\bz}^2\right)~.\label{eq:lowrank1}
  \end{align}
  As to the second claim in the lemma, let $G^{K}=USU^\top$ be the spectral decomposition of $G^{K}$ (where $U,S\in \reals^{2d\times 2d}$). Then we argue that $VDV^\top$ is a valid spectral decomposition of $K$, where $V\in \reals^{m\times 2d},D\in
  \reals^{2d\times 2d}$ are defined as
  \[
  \forall t=1,\ldots,m~~\forall j=1,\ldots,2d~~~~V_{t,j}= \sqrt{\frac{2d}{m}}U_{i(t),j}~,~ D_{j,j}=\frac{m}{2d}S_{j,j}~.
  \]
  This is because $V$'s columns are orthonormal (this can be easily checked
  based on $U$'s columns being orthonormal), and moreover, for any $t,t'\in
  \{1,\ldots,m\}$ such that $i(t)=j,i(t')=j'$, we have
  \begin{align*}
    K_{t,t'}&=G_{j,j'}=\sum_{p=1}^{2d}U_{j,p}S_{p,p}U_{j',p}\\
    &= \sum_{p=1}^{2d}\left(\sqrt{\frac{m}{2d}}~V_{t,p}\right)\left(\frac{2d}{m}D_{p,p}\right)\left(\sqrt{\frac{m}{2d}}~V_{t',p}\right)\\
    &= \sum_{p=1}^{2d} V_{t,p}D_{p,p} V_{t',p}
  \end{align*}
  so $K=VDV^\top$.

  Therefore, we get than any entry $t$ of $\balpha=\left(K+\frac{\lambda
      m}{2}I\right)^{-1}\by$ can be written as follows:
  \begin{align*}
     \alpha_t&=~\left(\left(K+\frac{\lambda m}{2}I\right)^{-1}\by\right)_t
      ~=~\left(\left(VDV^\top+\frac{\lambda m}{2}VIV^\top\right)^{-1} \by\right)_t\\
      &=\left(V\left(D+\frac{\lambda m}{2}I\right)^{-1}V^\top \by\right)_t
      ~=~ \sum_{p=1}^{2d}V_{t,p}\frac{1}{D_{p,p}+\frac{\lambda m}{2}}\left(\sum_{q=1}^{m}V_{q,p}\,y_q\right)\\
      &= \sum_{p=1}^{2d}\left(\sqrt{\frac{2d}{m}}U_{i(t),p}\right)\frac{1}{\frac{m}{2d}S_{p,p}
      +\frac{\lambda m}{2}}\left(\sum_{q=1}^{m}\left(\sqrt{\frac{2d}{m}}U_{i(q),p}\,z_{i(q)}\right)\right)\\
      &= \sum_{p=1}^{2d}\left(\sqrt{\frac{2d}{m}}U_{i(t),p}\right)\frac{1}{\frac{m}{2d}S_{p,p}
      +\frac{\lambda m}{2}}\left(\sum_{j=1}^{2d}\frac{m}{2d}\left(\sqrt{\frac{2d}{m}}U_{j,p}z_{j}\right)\right)\\
      &= \sum_{p=1}^{2d}U_{i(t),p}\frac{1}{\frac{m}{2d}S_{p,p}
      +\frac{\lambda m}{2}}\left(\sum_{j=1}^{2d}U_{j,p}z_{j}\right)
      ~=~ \frac{2d}{m}\sum_{p,j=1}^{2d}U_{i(t),p}\frac{1}{S_{p,p}
      +\lambda d}U_{j,p}z_{j}\\
      &=\frac{2d}{m}\bigl(U (S+\lambda d I)^{-1}U^\top\bz\bigr)_{i(t)}
      ~=~\frac{2d}{m}\left(\left(G^{K}+\lambda d I\right)^{-1}\bz\right)_{i(t)}
  \end{align*}
  from which it follows that for all $i=1,\ldots,2d$,
  \[
  \beta_{i}=\sum_{t\,:\,i(t)=i}\alpha_t =
  \frac{m}{2d}\times\frac{2d}{m}\left(\left(G^{K}+\lambda d I\right)^{-1}\bz\right)_{i} = \left(\left(G^{K}+\lambda d I\right)^{-1}\bz\right)_i~.
  \]
  Hence $\bbeta=\left(G^{K}+\lambda d I\right)^{-1}\bz$ as required.

\end{document}